\documentclass{article}

\usepackage[final]{neurips_2025}

\usepackage[utf8]{inputenc} 
\usepackage[T1]{fontenc}    
\usepackage{hyperref}       
\usepackage{url}            
\usepackage{booktabs}       
\usepackage{amsfonts}       
\usepackage{nicefrac}       
\usepackage{microtype}      
\usepackage{xcolor}         

\usepackage{xspace}
\usepackage{graphicx}
\usepackage{amssymb}
\usepackage{amsmath}
\usepackage{dsfont}
\usepackage{multirow}
\usepackage{float}
\usepackage{lipsum}
\usepackage{nicefrac}
\usepackage{circledsteps}
\usepackage{enumitem}
\usepackage{colortbl}
\usepackage{graphicx}
\usepackage{booktabs}
\usepackage{siunitx}
\usepackage{breqn}
\usepackage{lipsum}
\usepackage[capitalize]{cleveref}
\usepackage{duckuments}
\usepackage{tablefootnote}
\usepackage{pifont} 
\usepackage{makecell}
\usepackage{xr}
\usepackage{tabularx}
\usepackage{wrapfig}
\usepackage{soul}
\usepackage[outline]{contour}
\usepackage{amsthm}
\usepackage{algorithm}
\usepackage[noend]{algpseudocode}
\usepackage{bm}
\usepackage[frozencache,cachedir=.]{minted}
\usepackage{listings}
\usepackage{chngcntr}

\definecolor{LightGray}{gray}{1}

\usepackage{xargs}
\usepackage[colorinlistoftodos,prependcaption,textsize=tiny]{todonotes}
\newcommandx{\unsure}[2][1=]{\todo[linecolor=red,backgroundcolor=red!25,bordercolor=red,#1]{#2}}
\newcommandx{\change}[2][1=]{\todo[linecolor=blue,backgroundcolor=blue!25,bordercolor=blue,#1]{#2}}
\newcommandx{\info}[2][1=]{\todo[linecolor=OliveGreen,backgroundcolor=OliveGreen!25,bordercolor=OliveGreen,#1]{#2}}
\newcommandx{\improvement}[2][1=]{\todo[linecolor=Plum,backgroundcolor=Plum!25,bordercolor=Plum,#1]{#2}}
\newcommandx{\thiswillnotshow}[2][1=]{\todo[disable,#1]{#2}}

\newcommand{\sq}[1]{\textcolor{#1}{\rule{1.5ex}{1.5ex}}}

\newcommand{\sqfull}{\sq{taborange!70}}
\newcommand{\sqknots}{\sq{tabblue!90}}
\newcommand{\sqcore}{\sq{tabgreen!70}}

\newtheorem*{lemma}{Lemma}

\theoremstyle{definition}
\newtheorem{definition}{Definition}

\theoremstyle{definition}

\definecolor{lightgray}{gray}{0.95}
\definecolor{midgray}{gray}{0.55}
\definecolor{steelblue}{HTML}{4D82B7}
\definecolor{davysgrey}{rgb}{0.33, 0.33, 0.33}
\definecolor{LightCyan}{rgb}{0.88,1,1}
\definecolor{LightGold}{HTML}{F3E2C5}
\definecolor{AngelRow}{HTML}{FFFDD0}
\definecolor{ao(english)}{rgb}{0.0, 0.5, 0.0}
\definecolor{lightsalmon}{rgb}{1.0, 0.63, 0.48}
\definecolor{tabgreen}{HTML}{2CA02C}
\definecolor{tabblue}{HTML}{1F77B4}
\definecolor{taborange}{HTML}{FF7F0E}
\newcommand{\ourcolor}{tabgreen!15}

\newcommand{\deltatime}[1]{\textcolor{gray}{\textbf{\small #1}}}
\newcommand{\deltaneg}[1]{\textcolor{red}{\scriptsize{(#1)}}}
\newcommand{\deltapos}[1]{\textcolor{green!60!black}{\scriptsize{(#1)}}}
\newcommand{\deltazero}[0]{\textcolor{black}{\scriptsize{(+0.00)}}}

\usepackage[first=0, last=9]{lcg}

\newcommand{\vref}{V_{A}^{\mathrm{ref}}}
\newcommand{\uref}{U_{B}^{\mathrm{ref}}}

\DeclareMathOperator*{\argmin}{arg\,min}

\newcommand{\Star}[1]{#1\ensuremath{^*}\kern-\scriptspace}

\newcommand{\tit}[1]{\smallbreak\noindent\textbf{#1 }}

\newcommand{\tinytit}[1]{\noindent\textbf{#1}}

\crefname{section}{Sec.}{Secs.}
\crefname{table}{Tab.}{Tabs.}
\crefname{figure}{Fig.}{Figs.}

\newcommand{\ourrow}{\rowcolor{\ourcolor}[\dimexpr\tabcolsep+0.1pt\relax]}

\newcommand{\iso}{{\textcolor{gray}{\scriptsize\textbf{+Iso-C}}}}

\newcommand{\rightalign}[1]{\multicolumn{1}{r}{#1}}

\makeatletter
\DeclareRobustCommand\onedot{\futurelet\@let@token\@onedot}
\def\@onedot{\ifx\@let@token.\else.\null\fi\xspace}

\def\eg{\emph{e.g}\onedot} 
\def\ie{\emph{i.e}\onedot}

\def\wrt{w.r.t\onedot} 

\makeatother

\usepackage{array}
\newcommand{\PreserveBackslash}[1]{\let\temp=\\#1\let\\=\temp}
\newcolumntype{C}[1]{>{\PreserveBackslash\centering}p{#1}}
\newcolumntype{R}[1]{>{\PreserveBackslash\raggedleft}p{#1}}
\newcolumntype{L}[1]{>{\PreserveBackslash\raggedright}p{#1}}

\title{Accurate and Efficient Low-Rank \\ Model Merging in Core Space}

\author{%
Aniello Panariello$^{1}$\thanks{Equal Contribution. \texttt{aniello.panariello@unimore.it, daniel.marczak.dokt@pw.edu.pl}} \quad
Daniel Marczak$^{2,3\ *}$ \\[0.2cm]
\textbf{Simone Magistri}$^{4}$ \quad
\textbf{Angelo Porrello}$^{1}$ \quad
\textbf{Bart{\l}omiej Twardowski}$^{5,6}$ \\[0.2cm]
\textbf{Andrew D. Bagdanov}$^{4}$ \quad
\textbf{Simone Calderara}$^{1}$ \quad
\textbf{Joost van de Weijer}$^{6}$ \\ \\
$^{1}$AImageLab, University of Modena and Reggio Emilia, Italy \\
$^{2}$Warsaw University of Technology, Poland \quad
$^{3}$IDEAS NCBR, Warsaw, Poland \\
$^{4}$Media Integration and Communication Center (MICC), University of Florence, Italy \\
$^{5}$IDEAS Research Institute, Warsaw, Poland \\
$^{6}$Computer Vision Center, Universitat Aut\`{o}noma de Barcelona, Spain \\
}

\begin{document}

\maketitle

\begin{abstract}
In this paper, we address the challenges associated with merging low-rank adaptations of large neural networks. With the rise of parameter-efficient adaptation techniques, such as Low-Rank Adaptation (LoRA), model fine-tuning has become more accessible. While fine-tuning models with LoRA is highly efficient, existing merging methods often sacrifice this efficiency by merging fully-sized weight matrices. We propose the \textit{Core Space} merging framework, which enables the merging of LoRA-adapted models within a common alignment basis, thereby preserving the efficiency of low-rank adaptation while substantially improving accuracy across tasks. We further provide a formal proof that projection into Core Space ensures no loss of information and provide a complexity analysis showing the efficiency gains. Extensive empirical results demonstrate that Core Space significantly improves existing merging techniques and achieves state-of-the-art results on both vision and language tasks while utilizing a fraction of the computational resources. Codebase is available at \url{https://github.com/apanariello4/core-space-merging}.
\end{abstract}
\section{Introduction}

In recent years, the size of neural networks has grown substantially~\cite{falcon, brown2020language, deepseek-ai2024deepseekv3, grattafiori2024llama, qwen2.5}, increasing the economic and computational costs associated with training from scratch and fine-tuning. As a consequence, efficient low-rank adaptation techniques have emerged, which enable broader access to these powerful models~\cite{houlsby2019parameter, hu2021lora, kopiczko2023vera, yang2024dora, zhang2023adalora}. Techniques like Low-Rank Adaptation (LoRA)~\cite{hu2021lora} reparameterize model updates to significantly reduce the number of trainable parameters. This makes it feasible for a broader range of users to fine-tune large architectures on their specific tasks.

At the same time, the advent of model hubs such as Hugging Face~\cite{huggingface} has simplified the diffusion of pre-trained and fine-tuned models, opening new opportunities for collaborative and multi-task learning by allowing users to acquire and build upon existing models easily. In this context, model merging, which aims to combine multiple specialized models into one capable of handling various tasks, has been gaining interest~\cite{matena2021merging, rame2022diverse, wortsman2022model, mancusi2024is, pmlr-v274-soutif25a}. However, most prior works focus on fully fine-tuned models~\cite{choi2024revisiting,tsv, ilharco2023task,marczak2025task,ortizjimenez2023tangent,wang2024localizing,yadav2023tiesmerging,rinaldi2025update}. While this is practical for smaller architectures, fully fine-tuned versions of larger models are rare due to their high memory and compute costs. As model sizes grow, new strategies are needed to efficiently merge fine-tuned adaptations without incurring the prohibitive overhead of merging full models.

In~\cite{stoica2024knots}, the authors observe that directly applying existing merging techniques~\cite{ilharco2023task,yadav2023tiesmerging,yu2024language} to updates derived by multiplying low-rank components leads to suboptimal results. To address this, they introduce an alignment space that improves update compatibility. However, merging in this alignment space requires abandoning the low-rank representation and performing a singular value decomposition (SVD) on the horizontally concatenated full space updates. This suffers from two significant drawbacks: it eliminates the efficiency benefits of low-rank adaptation, and becomes prohibitively expensive as the base model size increases.
\begin{wrapfigure}{r}{0.51\textwidth}
    \centering

    \includegraphics[width=0.5\textwidth]{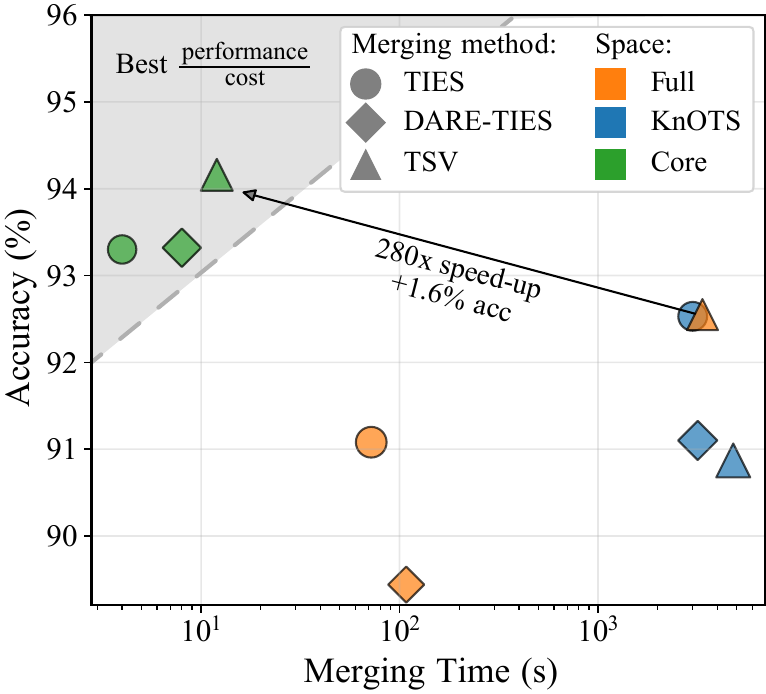}
    \caption{
    Merging in full space is fast but suboptimal (bottom center). Merging in KnOTS space or using strong merging methods (\eg, TSV) improves performance but increases cost by orders of magnitude (right). \textbf{Core Space merging is effective and efficient (top left).} Results on Llama 3 8B.
    }\label{fig:teaser}
    \vspace{-0.1in}
    \end{wrapfigure}

To overcome such limitations, we propose \textit{Core Space}, a novel parameter-efficient subspace that supports arbitrary merging techniques while retaining the benefits of low-rank adaptation. Core Space provides a common alignment basis for all task-specific low-rank components without loss of information. Notably, its dimensionality depends solely on the number of tasks and the LoRA rank, remaining tractable regardless of the base model size. Beyond its advantages in terms of efficiency, merging in Core Space also consistently improves the performance of existing merging strategies. We evaluate three setups: (1) \sqfull{} first multiplying low-rank matrices and then applying merging techniques, (2) \sqknots{} merging in the KnOTS space~\cite{stoica2024knots}, and (3) \sqcore{} merging in our proposed Core Space. Across both vision and language domains, merging in Core Space achieves the best results -- demonstrated on ViT-B/32, ViT-L/14, and Llama 3 8B backbones -- highlighting that our approach not only preserves parameter efficiency but also leads to improved generalization and task performance (see \cref{fig:teaser}).

The main contributions of this paper are the following:
\begin{itemize}[itemsep=.45em]
    \item We introduce \textit{Core Space Merging}, a framework to merge LoRA-adapted models in a shared low-rank basis, avoiding costly full space operations, while improving accuracy. Our approach can be easily integrated with existing merging methods.
    \item We prove that projection into Core Space ensures \textbf{no loss of information}, and provide a complexity analysis demonstrating the \textbf{efficiency gain} of merging in the proposed space.
    \item We present an extensive empirical evaluation showing \textbf{state-of-the-art results} achieved at a \textbf{fraction of the computational cost} of competing methods by merging in Core Space for vision and language tasks, including experiments on ViT-B/32, ViT-L/14, and Llama 3 8B.
\end{itemize}

\section{Related Work}

\tit{Parameter-efficient fine-tuning (PEFT).}
Pre-trained models serve as a starting point for training experts specialized in various downstream tasks~\cite{devlin2019bert,radford2021learning}. As the size of frontier models grows, the cost of fully fine-tuning such models increases accordingly. Therefore, several parameter-efficient fine-tuning (PEFT), updating a small fraction of parameters, have been proposed, including adapters~\cite{houlsby2019parameter}, prefix tuning~\cite{li2021prefix}, and prompt tuning~\cite{lester2021power}. Nowadays, LoRA~\cite{hu2021lora} and its variants~\cite{kopiczko2023vera,yang2024dora}, which rely on low-rank updates, have emerged as one of the most popular PEFT techniques.

\tit{Model merging.} The abundance of available expert models inspired the fundamental question behind model merging~\cite{matena2021merging}: \textit{how can we integrate knowledge from multiple expert models into a single multi-task model?} Task Arithmetic~\cite{ilharco2023task} proposed to construct \textit{task vectors} (\ie, the parameter difference between expert and base model) and aggregate them via scaled addition, creating a multi-task expert. Since a significant performance gap between the single-task and the merged models remains, many approaches were proposed to address this issue~\cite{Akiba2025evolutionary, jin2023regmean, marczak2024magmax, marouf2023weighted, tam2023merging, yang2024model, yang2024adamerging,panariello2025modular}. TIES-Merging~\cite{yadav2023tiesmerging} focuses on reducing sign conflicts between the parameters of expert models. Model Breadcrumbs~\cite{davari2023model} removes outliers from the task vectors, while Consensus Merging~\cite{wang2024localizing} eliminates catastrophic and selfish weights. Most recent methods, like TSV~\cite{tsv} and Iso-C~\cite{marczak2025task}, rely on singular value decomposition (SVD) of weight update matrices to reduce task interference when merging models. However, most methods are designed for merging fully fine-tuned models.

\tit{Merging LoRA-adapted models.}
Methods designed to merge fully fine-tuned models do not necessarily transfer well to merging LoRA-adapted models~\cite{tang2023parameter}. The authors of~\cite{tang2023parameter} proposed and improved a method to merge LoRAs. However, their approach relies on altering the fine-tuning procedure. KnOTS~\cite{stoica2024knots} proposes to merge LoRA updates in the shared subspace, achieving a significant improvement. However, KnOTS performs SVD on the concatenation of full-size matrices instead of leveraging their decomposed update representations, making it costly, especially for large weight matrices. Therefore, finding a method that effectively and efficiently merges LoRA-adapted models remains an open challenge.
\section{Preliminaries}
\tit{LoRA fine-tuning.} Low-Rank Adaptation (LoRA)~\cite{hu2021lora} is a technique for efficient fine-tuning of large pre-trained models. Instead of updating the full model weights $W \in \mathbb{R}^{m \times n}$, LoRA introduces two learnable matrices $A \in \mathbb{R}^{r \times n}$ and $B \in \mathbb{R}^{m \times r}$ (where $r \ll \min(m, n)$), and modifies the weight update as $W = W_0 + \Delta W = W_0 + BA$, where $W_0$ is the original weight matrix. This significantly reduces the number of parameters that need to be updated during fine-tuning.

\tit{Model merging.}
 Given a set of $T$ parameters $\{W_1, \dots, W_T\}$, for a common architecture trained on $T$ different tasks, a basic model merging approach calculates task vectors $\Delta W_i = W_i - W_0$, and computes a weighted sum $W_{\text{merged}} = W_0 + \alpha \sum_{i=1}^T \Delta W_i$. When dealing with LoRA-adapted models we obtain a common $W_0$ and a set of decomposed, low-rank updates ${\{\Delta W_i = B_i A_i\}}_{i=1}^{T}$. However, merging such weight matrices obtained from LoRA leads to suboptimal performance as shown in~\cite{stoica2024knots}, since LoRA-adapted models are less aligned \wrt to their fully fine-tuned counterparts.

\section{The Core Space Merging Framework}\label{sec:method}
In this section, we introduce \textit{Core Space Merging} (see \cref{fig:method}), a framework designed to identify an effective and efficient subspace -- referred to as the \textit{Core Space} -- in which model merging for LoRA-adapted models can be performed while remaining in the low-rank regime. Core Space is designed to be reversible -- it ensures no loss of information when projecting into Core Space and back to the original space -- while being as compact as possible. Compactness allows for the use of state-of-the-art merging methods relying on Singular Value Decomposition (SVD) of weight matrices, which are highly costly to perform in the original space for large models.
\begin{figure}
    \centering
    \includegraphics[width=\linewidth]{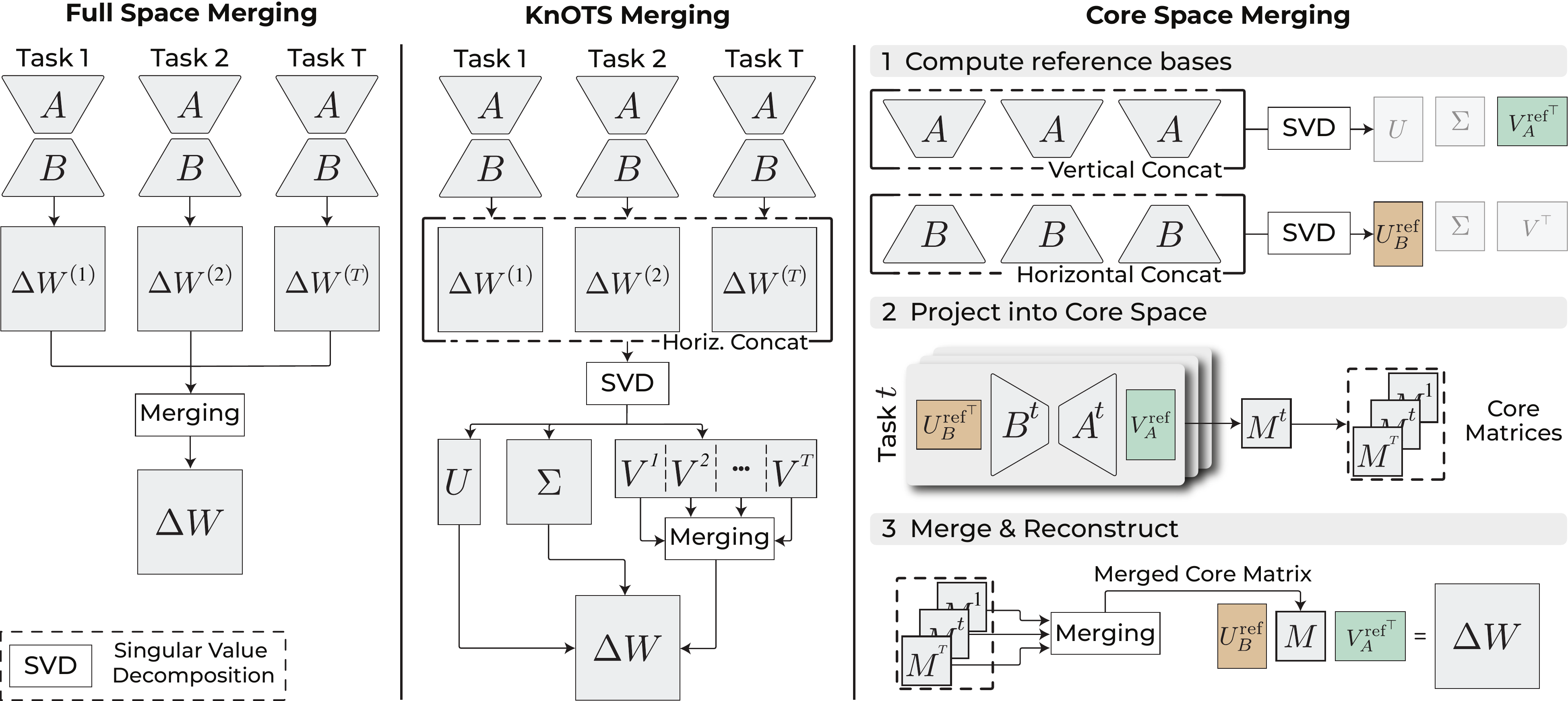}
    \caption{\textbf{Full Space Merging} (left) firstly reconstructs full space matrices $\Delta W^{(t)} = B^{(t)} A^{(t)}$, and then performs merging in the full space to obtain $\Delta W$. \textbf{KnOTS Merging} concatenates the $\Delta W^{(t)}$ matrices, and performs a costly SVD on this high-dimensional matrix. Then, the $V^{(t)}$ matrices are merged and used to obtain the final $\Delta W$. The proposed \textbf{Core Space Merging} (right) performs SVD on a concatenation of low-dimensional $A^{(t)}$ and $B^{(t)}$ matrices to obtain reference bases $(\vref, \uref)$. Afterwards, it projects each update into the Core Space to obtain the core matrices ${\{M^{(t)}\}}_{t=1}^T$. It then performs merging in the Core Space and reconstructs to obtain the final $\Delta W$.}\label{fig:method}
\end{figure}

\subsection{Model Merging in Core Space}\label{sec:core-space}
Let $A^{(t)} \in \mathbb{R}^{r \times n}$ and $B^{(t)} \in \mathbb{R}^{m \times r}$ denote the low-rank matrices for task $t$, derived from a shared pre-trained base model $W_0$. Each task update $\Delta W^{(t)} = B^{(t)} A^{(t)}$ can be reconstructed from the SVD of the matrices:
\begin{equation}
\label{lora-def}
\begin{gathered}
    A^{(t)} = U_{A}^{(t)} \Sigma_{A}^{(t)} V_{A}^{(t)\top}, \quad
    B^{(t)} = U_{B}^{(t)} \Sigma_{B}^{(t)} V_{B}^{(t)\top}, \\[.3em]
    \Delta W^{(t)} = U_{B}^{(t)} \Sigma_{B}^{(t)} V_{B}^{(t)\top}
                     U_{A}^{(t)} \Sigma_{A}^{(t)} V_{A}^{(t)\top},
\end{gathered}
\end{equation}
where the shapes of the matrices in the decomposition are: $ U_{A}^{(t)} \in \mathbb{R}^{r \times r}$, $\Sigma_{A}^{(t)} \in \mathbb{R}^{r \times r} $, $ V_{A}^{(t)} \in \mathbb{R}^{n \times r}$, $U_{B}^{(t)} \in \mathbb{R}^{m \times r}$, $\Sigma_{B}^{(t)} \in \mathbb{R}^{r \times r}$, and $V_{B}^{(t)} \in \mathbb{R}^{r \times r}$.

\tit{Motivation.} Under the hypothesis that all tasks share approximately the same common bases $(U_B, V_A)$ such that $\forall t \in \{1, \dots, T\}, \; U_B \approx U_{B}^{(t)}$ and $V_A \approx V_{A}^{(t)}$, we have:
\begin{equation}\label{eq:naive-merge}
     \Delta W = \sum_{t=1}^{T} \Delta W^{(t)} =\sum_{t=1}^{T} B^{(t)}A^{(t)}  \approx U_B \left( \sum_{t=1}^{T} \Sigma_{B}^{(t)} V_{B}^{(t)\top} U_{A}^{(t)} \Sigma_{A}^{(t)} \right) V_A^\top,
\end{equation}
where $\Sigma_{B}^{(t)} V_{B}^{(t)\top} U_{A}^{(t)} \Sigma_{A}^{(t)} \in \mathbb{R}^{r \times r}$ encodes the directional transformation applied by the low-rank update of task $t$. This suggests that, under aligned bases, the sum of low-rank updates (\ie, Task Arithmetic~\cite{ilharco2023task}) can be reduced to merging operations in a much smaller $r \times r$ space.

\tit{Projecting into the Core Space.} In practice, task-specific bases are not aligned, making direct merging of core matrices as in~\cref{eq:naive-merge} infeasible.
Therefore, we aim to find a shared basis that can represent all tasks without loss of information and that enables merging to be performed in a reduced space compared to the full space $\Delta W^{(t)}$. Intuitively, such a shared basis should span the subspace formed by the union of the individual task subspaces.

\begin{definition}[Reference Bases]
    Given a collection of low-rank matrices ${\{A^{(t)}, B^{(t)}\}}_{t=1}^T$, we define as \emph{reference bases} the orthonormal matrices $\uref\in \mathbb{R}^{m \times Tr}$ and $\vref\in \mathbb{R}^{n \times Tr}$, obtained by performing SVD over the vertically stacked $A^{(t)}$ and horizontally stacked $B^{(t)}$ low-rank matrices across tasks:
\begin{gather}
\label{eq:reference}
    [B^{(1)}; \ldots; B^{(T)}]= \uref \Sigma_{B} V_{B}^\top; \quad
    \begin{bmatrix}
    A^{(1)} \\
    \vdots \\
    A^{(T)}
    \end{bmatrix} = U_{A} \Sigma_{A} \left( \vref \right)^\top.
\end{gather}
These bases span a shared latent subspace into which all task-specific updates are projected.
\end{definition}
Although the reference bases $(\uref, \vref)$ span all task-specific directions, each task $t$ is originally expressed in its own local bases $(U_B^{(t)}, V_A^{(t)})$. To express each update in the common coordinate system for merging, we solve the following least-squares problems:
\begin{equation}
\label{eq:least-squares}
\begin{gathered}
R_B^{(t)} = \argmin_{R \in \mathbb{R}^{T \cdot r \times r}} \left\| \uref R -  U_B^{(t)} \right\|_F^2, \quad Q_A^{(t)} = \argmin_{Q \in \mathbb{R}^{T \cdot r \times  r}} \left\|  {\vref} Q - {V_A^{(t)}} \right\|_F^2,
\end{gathered}
\end{equation}
where $V_A^{(t)} \in \mathbb{R}^{n \times r}$ and $\vref \in \mathbb{R}^{n \times T \cdot r}$ (and similarly for $U_B^{(t)}$ and $\uref$). These problems are convex, and since $\uref$ and $\vref$ are orthonormal, setting the gradients to zero yields the global minimizers (see \cref{appendix-sec:least-squares} for the full derivation):
\begin{equation}
  R_B^{(t)} = {\uref}^\top U_B^{(t)}, \quad       Q^{(t)}_A = {\vref}^{\top} V_A^{(t)}\label{eq:procrustes}.
\end{equation}
As we will show in \cref{sec:theoretical_justification}, $\| \uref R_B^{(t)} -  U_B^{(t)} \|_F^2 = 0 $, which allows us to substitute $U_B^{(t)}$ with $\uref R_B^{(t)}$, and similarly $V_A^{(t)}$ with $\vref Q_A^{(t)}$, in~\cref{lora-def}, yielding:
\begin{equation}
\label{eq:delta_update}
    \Delta W^{(t)} = \uref \,\underbrace{R_B^{(t)}\Sigma_{B}^{(t)} V_{B}^{(t)\top}
    U_{A}^{(t)} \Sigma_{A}^{(t)} {Q_A^{(t)}}^\top}_{\text{task-$t$ update in reference coordinates}} {\vref}^{\top}.
\end{equation}
By substituting the least-squares solutions from \cref{eq:procrustes} and using the definitions of $B^{(t)}$ and $A^{(t)}$ from \cref{lora-def}, we can equivalently write:
\begin{equation}
\label{eq:delta_update_equivalent}
     \Delta W^{(t)} = \uref \left({\uref}^{\top} B^{(t)} A^{(t)} {\vref} \right) {\vref}^\top.
\end{equation}
 This reformulation expresses each $\Delta W^{(t)}$ in the reference basis, enabling all updates to be compared or merged within a shared coordinate system.
\begin{definition}[Core Matrix]
 We define the \emph{core matrix} $M^{(t)}$ as:
\begin{equation}
\label{eq:core}
M^{(t)} =  \left({\uref}^{\top} B^{(t)}\right) \left(A^{(t)} \vref\right) \in \mathbb{R}^{Tr \times Tr}.
\end{equation}
This formulation generalizes the middle expression in \cref{eq:naive-merge}, where aligned task-specific bases were implicitly assumed. In contrast, the core matrix $M^{(t)}$ is expressed in the reference bases $(\uref, \vref)$ and thus does not rely on any alignment assumption. It encodes the directional transformation applied by the low-rank update of task $t$, providing a compact and lossless representation of each task update in the shared reference space and enabling efficient merging in a reduced $Tr \times Tr$ space.
\end{definition}

\tit{Reparametrized Model Merging in Core Space.}
Once task-specific updates $\Delta W^{(t)}$ have been reparameterized into their corresponding core matrices $M^{(t)}$ in the shared reference bases $(\uref, \vref)$, Core Space Merging enables model merging to be performed entirely within a compact, aligned, low-rank subspace. Specifically, the merged update is computed by applying a merging operator $\mathcal{M}$ over the set of core matrices:
\begin{equation}
    M_{\text{merged}} = \mathcal{M}\big({\{M^{(t)}\}}_{t=1}^{T}\big),
\end{equation}
where $\mathcal{M}$ may be \textit{any merging function}, ranging from simple arithmetic averaging~\cite{ilharco2023task} to more advanced, non-linear or geometry-aware techniques~\cite{yadav2023tiesmerging,yu2024language}. The final update in the original model space, as described also in \cref{alg:core_merge}, is recovered by projecting $M_{\text{merged}}$ back through the reference bases:
\begin{equation}
    \Delta W = \uref M_{\text{merged}}{\vref}^\top.
\end{equation}

Because Core Space is a lossless representation of the original updates for each individual task (see \cref{eq:delta_update_equivalent}), merging in this space preserves all relevant task information. Furthermore, when $\mathcal{M}$ is \textit{linear}, such as Task Arithmetic, the merge operation in Core Space is \textit{exactly equivalent} to applying the same merge in the full model space:
\begin{equation}
\label{eq:exact-merging-linear}
    \mathcal{M}({\{ \Delta W^{(t)}\}}_{t=1}^T) =  \mathcal{M}({\{{\uref} M^{(t)} {\vref}^{\top}\}}_{t=1}^T) = \uref \mathcal{M}({\{M^{(t)}\}}_{t=1}^T ){\vref}^{\top}.
\end{equation}

Core Space merging offers key benefits over full space merging:
\begin{itemize}
    \item \textbf{Efficiency.} Core matrices $M^{(t)} \in \mathbb{R}^{Tr \times Tr}$ are significantly smaller than their full space counterparts $\Delta W^{(t)} \in \mathbb{R}^{m \times n}$. This reduction enables high-cost merging algorithms to run at a fraction of the time and memory footprint (see \cref{sec:complexity}).
    \item \textbf{Efficacy.} As shown in \cref{sec:results}, Core Space merging improves performance over full space merging when \textit{non-linear} methods are used. In \cref{sec:analyses}, we show that this improvement stems from better alignment and more compact representation of task-specific directions.
\end{itemize}
\subsection{No Information Loss in Core Space Representation}
\label{sec:theoretical_justification}

Replacing $U_B^{(t)}$ and $V_A^{(t)}$  with $\uref R_B^{(t)}$ and $\vref Q_A^{(t)}$ to obtain \cref{eq:delta_update,eq:delta_update_equivalent}, which define the final form of the core matrix, requires that the solutions to the least-squares problems in \cref{eq:least-squares} incur zero alignment error. That is,
\begin{equation}
  \left\| \uref R_B^{(t)} -  U_B^{(t)} \right\|_F^2 = 0, \quad \quad \left\|  {\vref} Q_A^{(t)} - {V_A^{(t)}} \right\|_F^2 = 0.
\end{equation}

In this section, we show that the reference bases $\uref$ and $\vref$, obtained via the SVD of the stacked matrices $B^{(t)}$ and $A^{(t)}$ (see \cref{eq:reference}), \textit{minimize} the total alignment error across all $T$ tasks, achieving an error of \textit{exactly zero}. To illustrate this, we first analyze the alignment error for a single task $t$, focusing on $\uref$. Analogous results hold symmetrically for $\vref$. For clarity, we assume in the following derivations that $T \cdot r \leq m$ and $T \cdot r \leq n$, so that the total LoRA rank does not exceed the maximum possible rank of the target weight matrix. In \cref{appendix-sec:overcomplete}, we provide a more general analysis that removes this assumption and demonstrate that the zero alignment error result continues to hold.
\begin{lemma}
Let $U_B^{(t)} \in \mathbb{R}^{m\times r}$ and $\uref \in \mathbb{R}^{m\times T \cdot r}$ be matrices with orthonormal columns, and let $R_B^{(t)} ={\uref}^\top U_B^{(t)} \in \mathbb{R}^{T \cdot r \times r}$ be the optimal solution minimizing the error of the least-square problem. Then, the optimal alignment error is given by:
\begin{equation}
\label{eq:singularerror}
    \varepsilon_U = \left\| \uref  R_B^{(t)} - U_B^{(t)}\right\|_F^2 = r - \left\| {U_B^{(t)}}^\top \uref \right\|_F^2.
\end{equation}
\end{lemma}

The proof, provided in \cref{appendix-sec:error}, leverages the properties of Frobenius norm and the orthonormality of $U_B^{(t)}$ and $\uref$. To formally demonstrate that our chosen reference basis $\uref$ minimizes the alignment error across all $T$ tasks (or equivalently maximize $||{U_B^{(t)}}^\top \uref ||_F^2$ for each task $t$), we first formulate the following constrained optimization problem for a single task, and then extend it to the multi-task scenario:
\begin{equation}
\label{eq:optimalreference}
\begin{gathered}
\max_{U \in \mathcal{S}} \left\| {U_B^{(t)}}^\top U \right\|_F^2
= \max_{U \in \mathcal{S} }\mathrm{tr} \left( U^\top {U_B^{(t)}} {U_B^{(t)}}^\top U \right), \\ \text{where}\quad  \mathcal{S} = \left\{ U \in \mathbb{R}^{m \times T r} \;\middle|\; U^\top U = I_{T \cdot r} \right\}
\end{gathered}
\end{equation}

and $\mathrm{tr}(\cdot)$ denotes the trace operator. The optimization domain is restricted to the Stiefel manifold $\mathcal{S}$ (\ie, the set of matrices with orthonormal columns). The following lemma characterizes the solution to the following optimization problem:

\begin{lemma}
A solution $U^*$ to the quadratic program in \cref{eq:optimalreference} is given by a basis whose columns include the $r$ eigenvectors corresponding to nonzero eigenvalues of $B^{(t)} {B^{(t)}}^{\top}\in \mathbb{R}^{m\times m}$ or, equivalently, by the $r$ left singular vectors of the matrix $B^{(t)}$. Moreover, at the optimum, the objective attains its maximum value $r$, resulting in zero alignment error in \cref{eq:singularerror}.
\end{lemma}

We refer the reader to \cref{appendix-sec:optimal-bases} for a detailed proof. Briefly, the result follows by applying the method of Lagrange multipliers to augment the optimization objective with the Stiefel manifold constraint and then enforcing stationarity by setting the gradient of the Lagrangian to zero.

\begin{algorithm}[t]
\caption{Core Matrix Alignment and Merging}
\label{alg:core_merge}
\begin{algorithmic}[1]
\Require Low-rank updates $\{(A^{(t)}, B^{(t)})\}_{t=1}^T$, merging function $\mathcal{M}(\cdot)$.

\State Stack $A^{(t)}$ vertically, $B^{(t)}$ horizontally
\State Compute $\vref$: $\text{stack}(A^{(t)}) = U_A \Sigma_A \underline{V_A^{\mathrm{ref}\top}}$ \Comment{reference bases}
\State Compute $\uref$: $\text{stack}(B^{(t)}) = \underline{U_B^{\mathrm{ref}}} \Sigma_B V_B^\top$
\For{$t = 1$ to $T$}
    \State Compute: $M^{(t)} = ({U_B^{\mathrm{ref}}}^\top B^{(t)}) (A^{(t)} V_A^{\mathrm{ref}})$ \Comment{\cref{eq:core}}
\EndFor

\State Merge aligned core matrices: $M_{\mathrm{merged}} = \mathcal{M}(\{M^{(t)}\}_{t=1}^T)$
\State \Return $\Delta W = U_B^{\mathrm{ref}} M_{\mathrm{merged}} V_A^{\mathrm{ref}\top}$ \Comment{reconstructed merged model}
\end{algorithmic}
\end{algorithm}

\tit{Extension to multiple tasks.} Achieving zero reconstruction error for a single model $t$ does not guarantee optimality for any other model $t' \neq t$. Therefore, we aim to identify a reference basis $U^{*}$ that jointly optimizes \cref{eq:optimalreference} across all $T$ models. We formulate this global problem as:
\begin{equation}
\label{eq:procrustesglobal}
 \max_{U \in \; \mathcal{S}} \sum_{t=1}^T \mathrm{tr}(U^\top {U_B^{(t)}} {U_B^{(t)}}^\top U) = \max_{U \in \; \mathcal{S}} \mathrm{tr}(U^{\top} \mathbf{U}_B \mathbf{U}_B^\top U),
\end{equation}
where $\mathbf{U}_B = \big[ U_B^{(1)}, \;\; U_B^{(2)}, \;\; \dots, \;\; U_B^{(T)} \big]$ denotes the horizontal concatenation of all $U_B^{(t)}$ matrices. The equality in \cref{eq:procrustesglobal} follows directly from the linearity of the trace operator and the distributivity of matrix multiplication concerning matrix addition: $M^\top A_1 M + M^\top A_2 M = M^\top (A_1 + A_2) M$.

By considerations analogous to the single task-case, a global solution $U^{*}$ is given by the top $T\cdot r$ left singular vectors of the matrix $\mathbf{B}$, obtained by vertically stacking each matrix $B^{(t)}$, \ie, $U^{*}=\uref$. This choice ensures zero alignment error simultaneously across all $T$ tasks, consistent with the procedure described in \cref{sec:core-space}.
\subsection{Computational Complexity Analysis}\label{sec:complexity}
\begin{figure}
    \centering
    \includegraphics[width=\linewidth]{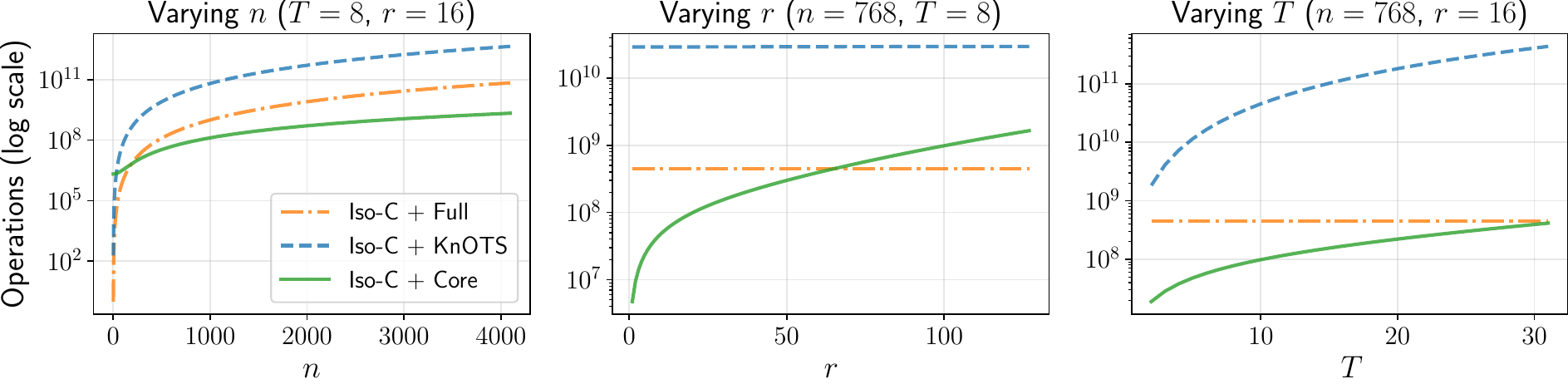}
    \caption{\textbf{Core Space merging is more efficient than the previous state-of-the-art KnOTS.} The cost is similar to full space merging, which results in much lower performance. We visualize the number of operations performed to merge $T$ rank $r$ LoRA modules of final shape $n \times n$.}\label{fig:complexity}
\end{figure}

\begin{wraptable}{r}{0.45\textwidth}
\centering
\vspace{-0.5em}
\caption{\label{table:complexity}$\mathcal{O}(\cdot)$ time complexities. The cheapest method is highlighted in \textbf{bold}  ($T, r \ll n$).}
\vspace{0.5em}
\resizebox{\linewidth}{!}{
\begin{tabular}{llll}
\toprule
\textbf{Space} & TA & Iso-C & TSV \\
\midrule
Full   & $\bm{n^2 T r}$ & $n^3$ & $n^3 T$ \\
KnOTS  & $n^3 T^2$ & $n^3 T^2 + n^2 T r$ & $n^3 T^2 + T^3 r^2 n$ \\
\ourrow
\textbf{Core}   & $\bm{n^2 T r}$ & $\bm{n^2 T r + T^3 r^3}$ & $\bm{n^2 T r + T^4 r^3}$ \\
\bottomrule
\end{tabular}}
\vspace{-0.5em}
\end{wraptable}
We summarize the time complexities of TA, Iso-C, and TSV merged in all three spaces (Full, KnOTS, and our Core) in \cref{table:complexity,fig:complexity}. Details on the derivation of the complexities are in Appendix~\ref{appendix-sec:computational}.
Our approach exhibits a time complexity comparable to that of Task Arithmetic in full space. Our method's additional terms are negligible unless the product $T \cdot r$ becomes significantly large. A key advantage of our method lies in its scalability compared to KnOTS, whose time complexity is super-cubic, driven by a factor that scales cubically with the weight matrix size $n$. Finally, we emphasize the minimal additional overhead incurred when combining our method with Iso-C or TSV in the Core Space; it introduces a cost substantially lower than its counterpart in full space or KnOTS space.

\section{Experimental Results}

\tinytit{Experimental setup.}
We follow the experimental setup of KnOTS and use the LoRA checkpoints provided by the authors~\citep{stoica2024knots}. For the vision experiments, we use two variants of CLIP~\citep{radford2021learning} with ViT-B/32 and ViT-L/14~\cite{dosovitskiy2021an} as vision encoders fine-tuned on a standard set of 8 tasks. We employ Llama 3 8B~\citep {grattafiori2024llama} fine-tuned on 6 NLI tasks for the language experiments. All models are fine-tuned with LoRA~\citep{hu2021lora} with rank 16 applied on all matrices (keys, queries, values, and outputs) across all attention layers.
Following~\cite{stoica2024knots}, we report \textit{normalized accuracy} as a ratio of the accuracy of the merged model on a given task to the accuracy of the model fine-tuned on this task. We also report \textit{absolute accuracy} for the joint-task setting (additional experimental details in \cref{appendix-sec:exp-details}).

\tinytit{Baseline merging spaces.} We compare our proposed Core Space with two alternative merging spaces. \textbf{Full Space} operates in space of full reconstructed weight matrices $\Delta W^{(t)} = B^{(t)}A^{(t)} \in \mathbb{R}^{m \times n}$. \textbf{KnOTS Space}~\cite{stoica2024knots} operates in the space of the right singular vectors of the concatenated reconstructed weight matrices ${\{ \Delta W^{(t)}\}}_{t=1}^{T} \in \mathbb{R}^{m \times nT}$.

\tinytit{Baseline merging methods.} We evaluate each merging space using the following merging methods. \textbf{Task Arithmetic (TA)}~\citep{ilharco2023task} performs a scaled summation of each task matrix $W_{\text{merged}} = W_0 + \alpha \sum_{i=1}^T \Delta W_i$. As this is a linear operation, the results of merging in each space are the same (see \cref{eq:exact-merging-linear} for Core and~\cite{stoica2024knots} for KnOTS). \textbf{TIES}~\citep{yadav2023tiesmerging} trims low-magnitude parameters and averages parameters with dominating sign, while \textbf{DARE}~\citep{yu2024language} preprocesses task vectors by randomly dropping a fraction of parameters and rescaling the remaining ones. \textbf{TSV}~\citep{tsv} concatenates low-rank approximations of task matrices and orthogonalizes them across tasks. \textbf{CART}~\cite{choi2024revisiting} calculates centered task vectors as a difference of fine-tuned weights from the average of all fine-tuned weights and performs task arithmetic on the low-rank approximation of these centered task vectors. \textbf{Iso-C}~\citep{marczak2025task} flattens the spectrum of singular values for a model merged with task arithmetic. As the spectrum flattening can be performed on weights merged with any merging technique, we combine Iso with other merging techniques, denoting it with \iso.
\subsection{Results}\label{sec:results}
\begin{table}[t]
\centering
\caption{Normalized accuracies of merged models on NLI tasks for Llama 3 8B.}
\resizebox{\textwidth}{!}{
\begin{tabular}{lL{1.5cm}ccccccccr}
\toprule
\textbf{Method} & \textbf{Space} & \textbf{SNLI} & \textbf{MNLI} & \textbf{SICK} & \textbf{QNLI} & \textbf{RTE} & \textbf{SCITAIL} & \textbf{Avg ($\Delta$Acc)} & \textbf{Time [s]} & \textbf{Rel. Time}\\
\midrule
\textit{Abs. Accuracy}  & & 92.50 & 90.31 & 91.58 & 94.49 & 89.86 & 96.52 & - & - & - \\
\midrule
TA  & Full                     & 93.57 & 95.28 & 87.96 & 68.71 & 100.0 & 96.73 & 90.38 \deltazero & \rightalign{9} & - \\
\midrule
\multirow{3}{*}{TIES} & Full & 95.17 & 96.19 & 84.18 & 74.18 & 100.0 & 96.78 & 91.08 \deltazero & \rightalign{72}& \deltatime{9} \\
& KnOTS              & 91.82 & 94.19 & 92.97 & 78.57 & 100.0 & 97.61 & 92.53 \deltapos{+1.45} & \rightalign{3000} & \deltatime{375} \\
\ourrow
\cellcolor{white} & \textbf{Core}               & 92.07 & 93.51 & 93.63 & 83.72 & 99.19 & 97.66 & \textbf{93.30 \deltapos{+2.22}} & \rightalign{\textbf{8}} & \deltatime{1} \\
\midrule
\multirow{3}{*}{DARE-TIES}  & Full             & 94.76 & 96.8 & 78.39 & 72.08 & 98.39 & 96.20 & 89.44 \deltazero & \rightalign{108} & \deltatime{13}\\
& KnOTS       & 91.62 & 96.72 & 74.90 & 84.75 & 99.48 & 99.13 & 91.10 \deltapos{+1.66} & \rightalign{3180} & \deltatime{397} \\
\ourrow
\cellcolor{white} & \textbf{Core}        & 92.10 & 93.58 & 93.70 & 83.68 & 99.19 & 97.66 & \textbf{93.32 \deltapos{+3.88}} & \rightalign{\textbf{8}} & \deltatime{1} \\
\midrule
\multirow{3}{*}{TSV}  & Full                   & 95.38 & 95.12 & 88.83 & 76.80 & 101.61 & 97.56 & 92.55 \deltazero & \rightalign{3360} & \deltatime{280}\\
& KnOTS              & 92.53 & 95.83 & 82.77 & 77.01 & 100.0 & 97.08 & 90.87 \deltaneg{-1.68} & \rightalign{4800} & \deltatime{400}\\
\ourrow
\cellcolor{white} & \textbf{Core}              & 95.86 & 95.70 & 89.25 & 83.89 & 102.42 & 97.86 & \textbf{94.16 \deltapos{+1.61}} & \rightalign{\textbf{12}} & \deltatime{1} \\
\midrule
\multirow{3}{*}{Iso-C} & Full                   & 55.00 & 39.04 & 76.54 & 55.90 & 46.77 & 69.25 & 57.08 \deltazero & \rightalign{540} & \deltatime{67} \\
& KnOTS              & 85.28 & 52.86 & 89.43 & 54.90 & 75.00 & 77.73 & 72.53 \deltapos{+15.45} & \rightalign{4860} & \deltatime{607}\\
\ourrow
\cellcolor{white} & \textbf{Core}                & 91.54 & 90.10 & 87.87 & 75.85 & 99.19 & 97.42 & \textbf{90.33 \deltapos{+33.25}} & \rightalign{\textbf{8}} & \deltatime{1} \\
\bottomrule
\end{tabular}}\label{tab:per-task-llama}
\end{table}

\tit{LLMs merging.} We present Llama 3 8B results in natural language inference in \Cref{tab:per-task-llama}. In line with our complexity analysis, merging in Core Space is much more efficient than merging in Full or KnOTS space, bringing up to $600 \times$ merging speed-up. Moreover, merging in Core Space improves the performance of all tested merging methods. In particular, it elevates TSV to 94.16\% average normalized accuracy, achieving state-of-the-art results.

\begin{table}[t]
\centering
\caption{Normalized accuracies of merged models on the vision tasks with ViT-B/32.}
\resizebox{\textwidth}{!}{
\begin{tabular}{lL{1.9cm}ccccccccc}
\toprule
\textbf{Method} & \textbf{Space} & \textbf{Cars} & \textbf{DTD} & \textbf{EuroSAT} & \textbf{GTSRB} & \textbf{MNIST} & \textbf{RESISC} & \textbf{SUN397} & \textbf{SVHN} & \textbf{Avg ($\Delta$ Acc)} \\
\midrule
\textit{Abs. Accuracy} & & 74.00 & 58.30 & 99.00 & 92.70 & 99.30 & 88.40 & 64.50 & 96.20 & - \\
\midrule
TA & Full & 81.97 & 73.72 & 48.97 & 42.24 & 53.12 & 71.50 & 97.46 & 41.25 & 63.78 \deltazero \\
\midrule
\multirow{3}{*}{TIES} & Full & 82.37 & 72.72 & 49.91 & 36.62 & 57.16 & 69.38 & 96.92 & 44.56 & 63.70 \deltazero \\
& KnOTS  & 83.75 & 74.45 & 50.36 & 47.31 & 67.01 & 71.79 & 96.51 & 50.64 & 67.73 \deltapos{+4.03} \\
\ourrow
\cellcolor{white} & \textbf{Core} & 84.74 & 76.46 & 52.19 & 50.41 & 67.36 & 71.21 & 96.45 & 50.18 & \textbf{68.63} \textbf{\deltapos{+4.93}} \\
\midrule
\multirow{3}{*}{DARE-TIES} & Full      & 82.14 & 73.72 & 49.35 & 37.78 & 56.63 & 70.14 & 97.35 & 42.12 & 63.65 \deltazero \\
& KnOTS & 82.01 & 72.90 & 44.15 & 45.54 & 60.59 & 70.89 & 95.56 & 47.64 & 64.91 \deltapos{+1.26} \\
\ourrow
\cellcolor{white} & \textbf{Core}  & 84.57 & 76.09 & 57.09 & 51.01 & 66.64 & 71.39 & 96.16 & 52.14 & \textbf{69.39} \textbf{\deltapos{+5.74}} \\
\midrule
\multirow{3}{*}{TSV} & Full       & 83.44 & 75.55 & 50.99 & 45.03 & 59.31 & 73.33 & 96.40 & 49.23 & \textbf{66.66} \textbf{\deltazero} \\
& KnOTS   & 81.86 & 74.91 & 51.25 & 41.64 & 53.93 & 71.64 & 97.95 & 40.36 & 64.19 \deltaneg{-2.47} \\
\ourrow
\cellcolor{white} & \textbf{Core}    & 83.86 & 75.09 & 52.64 & 45.39 & 58.53 & 72.95 & 97.63 & 45.21 & 66.41 \deltaneg{-0.25} \\
\midrule
\multirow{3}{*}{CART}  & Full       & 83.04 & 81.93 & 50.39 & 70.17 & 59.14 & 79.11 & 99.26 & 49.11 & 71.52 \deltazero \\
& KnOTS   & 83.94 & 75.18 & 52.23 & 54.48 & 64.78 & 74.48 & 95.88 & 55.73 & 69.59 \deltaneg{-1.93} \\
\ourrow
\cellcolor{white} & \textbf{Core}   & 80.83 & 83.94 & 54.99 & 73.28 & 66.25 & 80.95 & 98.69 & 48.57 & \textbf{73.44} \textbf{\deltapos{+1.92}} \\
\midrule
\multirow{3}{*}{TIES \iso} & Full & 78.86 & 74.45 & 60.01 & 39.02 & 66.65 & 70.30 & 98.39 & 48.59 & 67.03 \deltazero \\
& KnOTS & 78.46 & 80.38 & 58.81 & 64.97 & 72.10 & 76.89 & 98.33 & 49.78 & 72.47 \deltapos{+5.44} \\
\ourrow
\cellcolor{white} & \textbf{Core} & 82.91 & 84.76 & 52.41 & 78.79 & 71.56 & 81.43 & 99.48 & 52.14 & \textbf{75.44} \textbf{\deltapos{+8.41}} \\
\midrule
\multirow{3}{*}{DARE-TIES \iso} & Full & 78.71 & 75.54 & 50.84 & 42.86 & 65.03 & 71.88 & 98.92 & 48.08 & 66.48 \deltazero \\
& KnOTS & 82.93 & 74.18 & 49.31 & 46.73 & 66.64 & 71.82 & 96.72 & 50.57 & 67.36 \deltapos{+0.88} \\
\ourrow
\cellcolor{white} & \textbf{Core} & 83.27 & 83.12 & 54.55 & 79.04 & 71.83 & 82.08 & 99.36 & 52.37 & \textbf{75.70} \textbf{\deltapos{+9.22}} \\
\midrule
\multirow{3}{*}{TSV \iso} & Full  & 79.38 & 80.38 & 57.99 & 65.64 & 64.22 & 79.74 & 98.59 & 46.49 & 71.55 \deltazero \\
& KnOTS & 80.81 & 83.03 & 58.25 & 74.34 & 67.66 & 79.69 & 98.54 & 49.86 & 74.02 \deltapos{+2.47} \\
\ourrow
\cellcolor{white} & \textbf{Core} & 82.98 & 85.12 & 50.95 & 84.25 & 71.14 & 84.39 & 99.06 & 53.53 & \textbf{76.43} \textbf{\deltapos{+4.88}} \\
\midrule
\multirow{3}{*}{CART \iso}   & Full       & 80.33 & 82.11 & 57.31 & 77.38 & 71.17 & 81.57 & 98.72 & 51.91 & 75.06 \deltazero \\
&  KnOTS    & 82.05 & 80.47 & 56.12 & 64.58 & 62.40 & 78.81 & 99.22 & 45.05 & 71.09 \deltaneg{-3.97} \\
\ourrow
\cellcolor{white} & \textbf{Core}    & 82.93 & 84.21 & 51.14 & 81.32 & 72.12 & 82.83 & 99.33 & 55.32 & \textbf{76.15} \textbf{\deltapos{+1.09}} \\
\midrule
\multirow{3}{*}{Iso-C} & Full        & 80.16 & 83.03 & 51.44 & 74.76 & 70.72 & 79.89 & 98.66 & 50.20 & 73.60 \deltazero \\
& KnOTS   & 80.33 & 79.29 & 57.50 & 67.60 & 65.63 & 79.54 & 99.26 & 46.62 & 71.97 \deltaneg{-1.63} \\
\ourrow
\cellcolor{white} & \textbf{Core}    & 83.35 & 84.30 & 50.13 & 81.97 & 71.07 & 83.46 & 99.17 & 53.90 & \textbf{75.92} \textbf{\deltapos{+2.32}} \\
\bottomrule
\end{tabular}}\label{tab:per-task-vitb}
\end{table}

\tinytit{Per-task evaluation in vision setting.} We present per-task vision results for ViT-B/32 in \Cref{tab:per-task-vitb}. We observe that 8 out of 9 merging methods achieve their highest average accuracy when performed in our proposed Core Space. The best combination -- TSV + Iso-C merged in Core Space -- achieves state-of-the-art average normalized accuracy of 76.3\%. It significantly outperforms the previously reported SoTA of TIES in KnOTS space, achieving 68.0\%~\cite{stoica2024knots}. Similar conclusions hold for experiments on ViT-L/14 presented in \cref{appendix-sec:vitl}.

\tinytit{Heterogeneous ranks.} While handling LoRA modules with heterogeneous ranks might seem non-trivial, our method supports it seamlessly without modification. Even with different ranks, the modules can be concatenated across tasks to form an aggregate basis spanning the combined subspaces, after which projection and alignment are applied to each local task core matrix. Since SVD makes no assumptions about input ranks, it yields valid orthonormal bases in all cases, enabling our framework to merge variable-rank LoRA modules naturally. We evaluate this setting by assigning rank 16 to half the tasks and rank 64 to the rest; the results reported in \cref{appendix-sec:heterogeneous-ranks} show that our method still outperforms other approaches.

\tinytit{Additional PEFT methods} Our method can also be applied to other PEFT methods, such as VeRA~\cite{kopiczko2023vera}. In VeRA, $\Delta W = \Lambda_b B \Lambda_d A$, where $A \in \mathbb{R}^{r \times n}$, $B \in \mathbb{R}^{m \times r}$, $\Lambda_b \in \mathbb{R}^{1\times m}$, and $\Lambda_d \in \mathbb{R}^{r\times1}$. Unlike LoRA, in VeRA the $A$ and $B$ matrices are randomly chosen, frozen, and shared across the network, while only the two scaling vectors $\Lambda$ are learned for each layer. To adapt VeRA to our Core Space merging, we absorb the scaling vectors into the matrices, \ie, $\Tilde{B}=\Lambda_b B$ and $\Tilde{A}=\Lambda_d A$, and then treat $\Tilde{A}$ and $\Tilde{B}$ as the LoRA $A$ and $B$ matrices. To confirm that our method also works with VeRA, we report additional experiments in \cref{appendix-sec:vera}, which show that our method outperforms other baselines also in this setting.

\begin{table}[t]
\centering
\small
\caption{Joint-task setting absolute accuracy of merged models on the vision tasks with ViT-B/32.}
\resizebox{\textwidth}{!}{
\begin{tabular}{lcccccccccc}
\toprule
\multirow{2}{*}{\textbf{Space}}  & \multirow{2}{*}{TA} & \multirow{2}{*}{TIES} & \multirow{2}{*}{DARE-TIES} & \multirow{2}{*}{TSV} & \multirow{2}{*}{CART} & \multirow{2}{*}{\shortstack{TIES\\\iso}} & \multirow{2}{*}{\shortstack{DARE-TIES\\\iso}} & \multirow{2}{*}{\shortstack{TSV\\\iso}} & \multirow{2}{*}{\shortstack{CART\\\iso}} & \multirow{2}{*}{Iso-C} \\ \\
\midrule
Full & 43.5 & 43.6 & 44.0 & \textbf{45.4} & 44.8 & 43.5 & 44.3 & 48.3 & 44.8 & 52.1 \\
KnOTS & 43.5 & 46.8 & 45.2 & 44.6 & 44.7 & 40.5 & 44.8 & 51.4 & 52.6 & 52.9 \\
\ourrow
\textbf{Core} & 43.5 & \textbf{47.4} & \textbf{47.6} & 44.5 & \textbf{49.6} & \textbf{54.1} & \textbf{54.0} & \textbf{55.7} & \textbf{55.6} & \textbf{55.9} \\
\bottomrule
\end{tabular}
}\label{tab:joint-vitb}
\end{table}

\tinytit{Joint-task evaluation in vision setting.} We also evaluate vision models in the challenging joint-task setting introduced in~\cite{stoica2024knots}, in which the task ID is unknown during inference. Instead of performing a multi-task evaluation, it evaluates the merged model on the union of all classes, requiring the model to distinguish between classes from all tasks. We present the results in \Cref{tab:joint-vitb}. Core Space facilitates merging with almost all methods, achieving state-of-the-art results when combined with Iso-C.
\subsection{Analysis}\label{sec:analyses}
\begin{figure}[t]
  \begin{minipage}{.49\linewidth}
    \centering
    \includegraphics[width=1.0\textwidth]{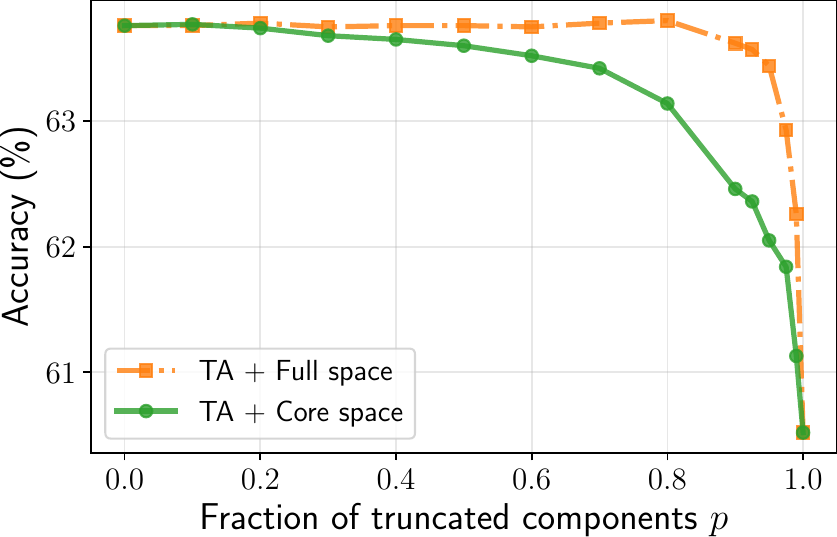}
    \caption{
        Most components in full space are irrelevant when doing Task-Arithmetic (TA). Removing any components from the Core Space results in a performance drop, showing that it is an information-dense space. We report the results on vision tasks with ViT-B/32.
    }\label{fig:truncation}
  \end{minipage}
  \hfill
  \begin{minipage}{.49\linewidth}
    \centering
    \includegraphics[width=1.0\textwidth]{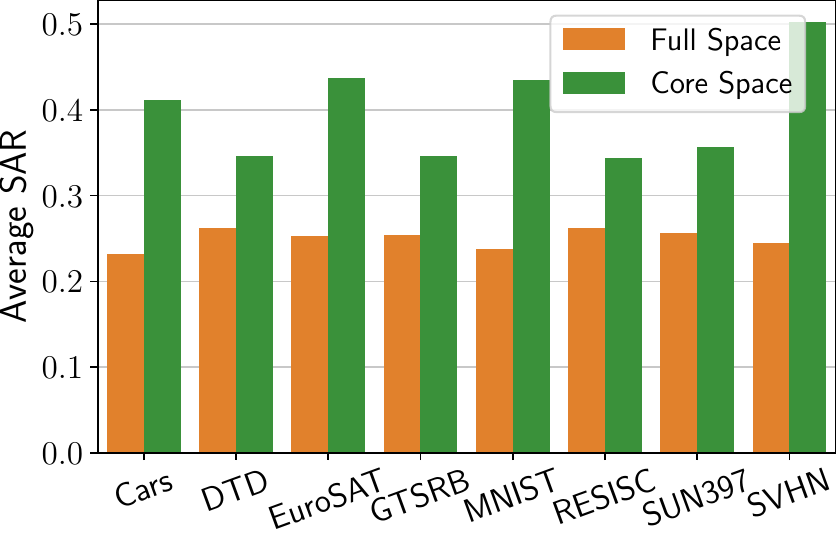}
    \caption{\textbf{Subspace Alignment Ratio} (SAR)~\cite{marczak2025task}. Each bar shows the average SAR between LoRA task matrices, Full, and Core Space. In Core Space, task matrices exhibit higher SAR. The associated performance gains suggest that better alignment facilitates more effective merging.}\label{fig:sar}
  \end{minipage}
\end{figure}
\tinytit{Truncation.} We herein compare the utilization of full space and Core Space for models merged with TA. Firstly, we calculate the SVD of the merged matrices: $\Delta W_{\text{merged}}$ for full space and $M_{\text{merged}}$ for Core Space. Then, we truncate a fraction of $p$ least significant values, \ie, $\sigma_i = 0$ for $i > (1-p) * \text{dim}(\Sigma)$, and observe a drop in the accuracy of the merged model after truncation. As shown in \Cref{fig:truncation}, in full space we can truncate a fraction up to $p=0.8$ values without performance loss, while in Core Space, truncation of any component results in a performance drop. It shows that the Core Space is dense while the full space contains many unused or redundant components. We hypothesize that the compactness of Core Space facilitates model merging as it extracts only the relevant components.

\tinytit{Core Space improves subspace alignment.} In this Section, we evaluate the Subspace Alignment Ratio (SAR)~\cite{marczak2025task} between each pair of LoRA updates fine-tuned on different tasks. The SAR measures how much of the subspace of one task is contained in another and correlates with post-merge performance. We compute SAR in full and Core Space. \Cref{fig:sar} shows that Core Space yields consistently higher alignment. We argue that this result is because the Core Space enforces a shared basis across tasks, which filters out task-specific noise and promotes alignment. In~\cref{sec:app_interference} we show that higher SAR correlates with lower merging interference.

\paragraph{Choice of the reference basis.}
\begin{wraptable}{r}{0.6\textwidth}
\centering
\vspace{-0.25in}
\caption{Ablations on the choice of reference basis. Our basis (3) achieves higher results than the single-task basis (1) and the random orthonormal basis of the same dimensionality (2). We proceed with $\vref$ analogously to $\uref$. We report the TIES-Core results on vision tasks with ViT-B/32.}
\vspace{1em}
\resizebox{\linewidth}{!}{
    \begin{tabular}{L{.2cm}llcc}
    \toprule
    & \textbf{Reference Basis $\uref$} & \textbf{Shape} & \textbf{Avg. Acc.} & \textbf{Avg. $\varepsilon_U$} \\
    \midrule
    (1) & $U^{(1)}_B$ (first task)      & $m \times r$  & 60.4 & 13.4\\
    (2) & Random orthonormal            & $m \times Tr$ & 61.6 & 13.3 \\
    \ourrow
    (3) & Concatenation (\cref{eq:reference}) & $m \times Tr$ & \textbf{68.6} & \textbf{0.0} \\
    \bottomrule
    \end{tabular}
}
\label{tab:ref_basis_ablation}
\end{wraptable}
To evaluate the reference bases choice, we assess the performance when adopting different reference bases, and compute the alignment error $\varepsilon_U$ defined in \cref{eq:singularerror} (averaged over all layers and tasks).
We present the results in \Cref{tab:ref_basis_ablation}. In row (1), we evaluate using the basis of the first task as a reference basis. In row (2), we set the reference basis to a random orthonormal basis of the same dimensionality. These two bases perform much less than our reference basis in row (3). Moreover, we confirm that the optimal reference basis from row (3) achieves zero alignment error. Additionally, we verified experimentally that even in the extreme case where $T \cdot r > \min(m,n)$ (\eg, $Tr = 2048 > 768$ for merging 8 ViT-B/32 LoRA models), the reconstruction error defined in \cref{eq:singularerror} remains exactly zero, consistent with the generalized theoretical result in \cref{appendix-sec:overcomplete}.

\section{Conclusion}
We propose Core Space, an efficient and effective method for merging LoRA modules. By projecting task-specific LoRA updates into a common subspace, Core Space reduces alignment error, leading to consistent accuracy improvements and SOTA results in both vision and language settings, while remaining computationally efficient. Our evaluations across vision and language domains confirm its scalability and strong performance in practical settings. We believe that Core Space can contribute to more efficient and accessible model adaptation in multi-task settings, particularly for large models.
\newpage
\subsubsection*{Acknowledgments}

We acknowledge the Spanish project PID2022-143257NB-I00, financed by MCIN/AEI/10.13039/501100011033 and ERDF/EU, and Funded by the European Union ELLIOT project. This work was supported by the Fortissimo Plus (FFplus) project Grant Agreement No. 101163317), under the European High‑Performance Computing Joint Undertaking (EuroHPC JU) and the Digital Europe Programme. The authors gratefully acknowledge access to compute resources enabled by FFplus. Funded by the European Union. Aniello Panariello acknowledges travel support from ELIAS (GA no 101120237). Daniel Marczak is supported by National Centre of Science (NCN, Poland) Grant No. 2021/43/O/ST6/02482. Bart{\l}omiej Twardowski acknowledges the grant RYC2021-032765-I and National Centre of Science (NCN, Poland) Grant No. 2023/51/D/ST6/02846.

\bibliography{bib.bib}
\bibliographystyle{abbrv}
\clearpage
\appendix

\section*{Appendix}

\section*{Contents of the Appendix}
\begin{itemize}
  \item A \quad No Information Loss in Core Space Representation
  \begin{itemize}
    \item A.1 \quad Least Squares Solutions
    \item A.2 \quad Alignment Error Quantification
    \item A.3 \quad Optimal Reference Bases
    \item A.4 \quad Generalization Beyond the $T \cdot r \leq m,n$ Assumption
  \end{itemize}
  \item B \quad Computational Complexity Analysis
  \item C \quad Additional Analysis
  \begin{itemize}
      \item C.1 \quad High subspace alignment leads to lower interference
      \item C.2 \quad Rank of the merged update matrices
  \end{itemize}
  \item D \quad Additional Experiment Details
  \begin{itemize}
      \item D.1 \quad Experimental Environment
      \item D.2 \quad Hyperparameter Search
      \item D.3 \quad Code Availability and Pseudocode
  \end{itemize}
  \item E \quad Additional Results
    \begin{itemize}
      \item E.1 \quad Per-task evaluation in vision setting for ViT-L/14
      \item E.2 \quad Experiments with Heterogeneous Ranks
      \item E.3 \quad Extension to VeRA
  \end{itemize}
\end{itemize}

\section{No Information Loss in Core Space Representation}\label{appendix-sec:no-info-loss}

\subsection{Least Squares Solutions}\label{appendix-sec:least-squares}
Let \( U_{A}^{(t)} \in \mathbb{R}^{r \times r} \), \( \Sigma_{A}^{(t)} \in \mathbb{R}^{r \times r} \), \( V_{A}^{(t)} \in \mathbb{R}^{n \times r} \), \( U_{B}^{(t)} \in \mathbb{R}^{m \times r} \), \( \Sigma_{B}^{(t)} \in \mathbb{R}^{r \times r} \), and \( V_{B}^{(t)} \in \mathbb{R}^{r \times r} \) be the components of two low-rank matrices $A^{(t)}$ and $B^{(t)}$ represented in SVD form. Let  $\uref \in \mathbb{R}^{m \times T\cdot r}$ and $\vref \in \mathbb{R}^{n \times T\cdot r}$ be the shared \textit{reference bases}, obtained by taking the left and the right singular vectors from the SVD of the horizontally and vertically stacked low-rank matrices $B^{(t)}$ and $A^{(t)}$, respectively (see Eq.~4 in the main paper). We assume $T\cdot r \leq m$ and $T\cdot r \leq n$, where $T$ is the number of tasks and $r$ is LoRA rank, as both are typically small relative to the feature dimension. Then, the solutions to the least square problems:
\begin{equation}
\label{eq:app-least-squares}
\begin{gathered}
R_B^{(t)} = \argmin_{R \in \mathbb{R}^{T \cdot r \times r}} \left\| \uref R -  U_B^{(t)} \right\|_F^2, \quad Q_A^{(t)} = \argmin_{Q \in \mathbb{R}^{T \cdot r \times  r}} \left\|  {\vref} Q - {V_A^{(t)}} \right\|_F^2,
\end{gathered}
\end{equation}
are given by:
\begin{equation}
  R_B^{(t)} = {\uref}^\top U_B^{(t)}, \quad Q^{(t)}_A = {\vref}^{\top} V_A^{(t)}.
  \label{eq:app-procrustes}
\end{equation}

\begin{proof}
Since $V_A^{(t)}$ and $U_B^{(t)}$ come from the SVDs of $A^{(t)}$ and $B^{(t)}$, they are orthonormal:
\begin{equation*}
 {V_A^{(t)}}^{\top} V_A^{(t)}= I_r \qquad {U_B^{(t)}}^{\top} U_B^{(t)} = I_r,
\end{equation*}
where $I_r$ is the $r \times r$ identity matrix. Similarly, the reference bases $\uref$ and $\vref$ also have orthonormal columns:
\begin{equation*}
    {\vref}^{\top} \vref = I_{T \cdot r}, \qquad {\uref}^{\top}\uref = I_{T \cdot r},
\end{equation*}
where $I_{T \cdot r}$ is the $T \cdot r \times T \cdot r$ and $T \cdot r \times T \cdot r$ identity matrix.

Consider the problems in \cref{eq:app-least-squares}. Each objective is convex and admits a unique global minimizer, as $\uref$ and $\vref$ have full column rank due to their orthonormality. To solve the first problem, we compute the gradient of the objective function with respect to $R$:
\begin{equation*}
    \frac{\partial \left\| \uref R -  U_B^{(t)} \right\|_F^2}{\partial R} =  2 {\uref}^{\top}(\uref R -  U_B^{(t)} ).
\end{equation*}
Setting the gradient to zero and solving the resulting equation gives:
\begin{equation*}
    {\uref}^{\top} \uref R_B^{(t)} = {\uref}^{\top} U_B^{(t)} \quad \Rightarrow \quad R_B^{(t)} =  {\uref}^\top U_B^{(t)},
\end{equation*}
since ${\uref}^{\top}\uref = I_{T \cdot r}$.

Similarly, for the second problem in \cref{eq:app-least-squares}:
\begin{equation*}
    Q_A^{(t)} = {\vref}^{\top} V_A^{(t)}
\end{equation*}
\end{proof}
\subsection{Alignment Error Quantification}\label{appendix-sec:error}

\begin{lemma}
Let $U_B^{(t)} \in \mathbb{R}^{m\times r}$ and $\uref \in \mathbb{R}^{m\times T \cdot r}$ be matrices with orthonormal columns where $T \cdot r = \mathrm{rank}([B^{(1)}, \dots, B^{(T)}]) \leq m$, and let $R_B^{(t)} ={\uref}^\top U_B^{(t)} \in \mathbb{R}^{T \cdot r \times r}$ be the optimal solution minimizing the error of the least-square problem. Then, the optimal alignment error is given by:
\begin{equation}
\label{eq:app-singularerror}
    \varepsilon_U = \left\| \uref  R_B^{(t)} - U_B^{(t)}\right\|_F^2 = r - \left\| {U_B^{(t)}}^\top \uref \right\|_F^2.
\end{equation}
\end{lemma}

\begin{proof}
To derive the alignment error, we simplify the notation by temporarily omitting the dependency on the model index $t$. Substituting the definition $R_B = {\uref}^\top U_B$ from \cref{eq:app-procrustes}, we can write:
\begin{align*}
\varepsilon_U &=\vert \vert \uref {\uref}^{\top} U_B - U_B \vert \vert_F^2 =  \vert \vert (\uref {\uref} ^{\top} - I_m) U_B \vert \vert_F^2  \\
&=\mathrm{tr} \left( {\left[ (\uref {\uref}^\top - I_m) U_B \right]}^\top \left[ (\uref {\uref}^\top - I_m) U_B \right] \right)   \\
& = \mathrm{tr} \left( U_B^{\top} {(\uref {\uref}^\top - I_m)}^2 U_B \right) \\
&= \mathrm{tr} \left( U_B^{\top} ( \uref {\uref}^\top \uref {\uref}^\top - 2 \uref {\uref}^\top + I_m   ) U_B \right) \\
&= \mathrm{tr} \left( U_B^{\top} ( \uref {\uref}^\top - 2 \uref {\uref}^\top + I_m   ) U_B \right) \quad\quad (\text{by using }\ {\uref}^\top \uref = I) \\
&=  \mathrm{tr} \left( U_B^\top ( I_m - \uref {\uref}^\top) U_B \right)=\\
&= \mathrm{tr} \left({U_B}^{\top} U_B\right) - \mathrm{tr}\left( U_B^{\top} \uref  {\uref}^{\top} U_B \right)
\quad\quad\quad\quad\ \text{(by linearity  of trace)} \\
&= r - \left\| U_B^{\top} \uref \right\|_F^2,
\end{align*}
where $\mathrm{tr}(U_B^\top U_B)=\|U_B\|_F^2=r$, since $U_B$ has orthonormal columns, and $\mathrm{tr}(U_B^\top \uref {\uref}^\top U_B) = \| U_B^\top \uref \|_F^2$ by the cyclic property of the trace and definition of the Frobenius norm.
\end{proof}
\subsection{Optimal Reference Bases}\label{appendix-sec:optimal-bases}

\begin{lemma}
A solution $U^*$ to the quadratic program:
\begin{equation*}
\begin{aligned}
\max_{U \in \mathcal{S}} \left\| {U_B^{(t)}}^\top U \right\|_F^2
= \max_{U \in \mathcal{S} }\mathrm{tr} \left( U^\top {U_B^{(t)}} {U_B^{(t)}}^\top U \right), \quad \mathcal{S} = \left\{ U \in \mathbb{R}^{m \times T \cdot r} \;\middle|\; U^\top U = I_{T \cdot r} \right\},
\end{aligned}
\end{equation*}
  is given by any orthonormal basis whose columns include the top $r$ eigenvectors corresponding to nonzero eigenvalues of $B^{(t)} {B^{(t)}}^{\top}\in \mathbb{R}^{m\times m}$ or, equivalently, by the top $r$ left singular vectors of $B^{(t)}$. At the optimum, the objective achieves the maximum value of $r$, yielding zero alignment error
  in $\varepsilon_U$ as defined in \cref{eq:app-singularerror}.
\end{lemma}

\begin{proof}
The proof proceeds in two steps. First, we show that any orthonormal basis $U_*$ containing the top $r$ eigenvectors of $U_B^{(t)} U_B^{(t)\top}$ solves the optimization problem. Second, we establish that the eigenvectors of $U_B^{(t)} U_B^{(t)\top}$ are the same as those of $B^{(t)} B^{(t)\top}$.

\tit{Step 1: Solving the constrained optimization.} By leveraging the method of Lagrange multipliers, we write the augmented objective function:
\begin{equation}
 \mathcal{L}(U, \Lambda) = \mathrm{tr}(U^\top U_B^{(t)} U_B^{(t)\top} U) - \mathrm{tr}\left(\Lambda (U^\top U - I_{T \cdot r})\right),
\end{equation}
where $\Lambda \in \mathbb{R}^{T \cdot r \times T \cdot r}$ is a matrix of Lagrange multipliers. Taking gradients with respect to $U$ and $\Lambda$, we obtain:
\begin{equation}
\begin{aligned}
\nabla_{U} \mathcal{L}(U, \Lambda)
&= \frac{\partial}{\partial U} \mathrm{tr}(U^\top U_B^{(t)} U_B^{(t)\top} U)
   - \frac{\partial}{\partial U} \mathrm{tr}(\Lambda (U^\top U - I_{T \cdot r})) \\
&= 2 U_B^{(t)} U_B^{(t)\top} U - 2 U \Lambda
\end{aligned}
\end{equation}
\begin{align}
\nabla_{\Lambda} \mathcal{L}(U, \Lambda)
= -\frac{\partial}{\partial \Lambda} \mathrm{tr}(\Lambda (U^\top U - I_{T \cdot r}))
= U^\top U - I_{T \cdot r}.
\end{align}

Setting the gradients to zero gives the two necessary optimality conditions:
\begin{equation}
    \label{eq:first-condition}
    U_B^{(t)} U_B^{(t)\top} U_* = U_* \Lambda_*,
\end{equation}
\begin{equation}
  \label{eq:second-condition}
    U_*^\top U_* = I_{T \cdot r}.
\end{equation}


The second condition, \cref{eq:second-condition}, holds by construction: we explicitly choose $U_* \in \mathbb{R}^{m \times T \cdot r}$ to be an orthonormal matrix. Specifically, we define $U_* = [v_1, \ldots, v_r, p_1, \ldots, p_{(T-1)\cdot r}]$, where $\{v_i\}_{i=1}^r$ are orthonormal eigenvectors of $U_B^{(t)} U_B^{(t)\top}$ associated with its nonzero eigenvalues $\lambda_1, \ldots, \lambda_r$, and $\{p_j\}_{j=1}^{(T-1)\cdot r}$ are additional orthonormal vectors chosen to complete the basis.

Next, to verify the first condition reported in \cref{eq:first-condition}, we define the diagonal matrix $\Lambda_* = \mathrm{diag}(\lambda_1, \ldots, \lambda_r, 0, \ldots, 0) \in \mathbb{R}^{T \cdot r \times T \cdot r}$, where the trailing zeros correspond to the eigenvalues associated with the orthogonal complement. Since each $v_i$ is an eigenvector of $U_B^{(t)} U_B^{(t)\top}$ with eigenvalue $\lambda_i$, and each $p_j$ lies in the nullspace of that matrix, we have:
\begin{equation*}
U_B^{(t)} U_B^{(t)\top} v_i = \lambda_i v_i \quad \text{for } i = 1, \dots, r,
\quad \text{and} \quad
U_B^{(t)} U_B^{(t)\top} p_j = 0 \quad \text{for } j = 1, \dots, (T-1)\cdot r.
\end{equation*}

Therefore:
\[
U_B^{(t)} U_B^{(t)\top} U_* = [\lambda_1 v_1, \dots, \lambda_r v_r, 0, \dots, 0] = U_* \Lambda_*,
\]

which confirms that the first-order condition in \cref{eq:first-condition} is satisfied.

Next, we substitute this result into the original expression to obtain the optimal value:
\[
\mathcal{L}^{\ast} = \mathrm{tr}(U_*^\top U_B^{(t)} U_B^{(t)\top} U_*) = \mathrm{tr}(U_*^\top U_* \Lambda_*) = \mathrm{tr}(\Lambda_*) = \sum_{i=1}^r \lambda_i = \mathrm{tr}(U_B^{(t)} U_B^{(t)\top}) = r,
\]
proving that the maximum value of the quadratic problem is $r$, and the corresponding alignment error $\varepsilon_U$ for $U_*$ is zero.

\tit{Step 2: Equivalence of eigenspaces.} Now we show that, given the matrix $B^{(t)} \in \mathbb{R}^{n \times r}$ from a LoRA-based adaptation, if $v$ is an eigenvector of $B^{(t)} B^{(t)\top}$, then $v$ is also, by construction, an eigenvector of $U_B^{(t)} U_B^{(t)\top}$, where $B^{(t)} = U_B^{(t)} \Sigma^{(t)} V_B^{(t)\top}$:
\begin{align*}
    v \text{ is an eigenvector of } B^{(t)} B^{(t)\top} &\implies B^{(t)} B^{(t)\top} v = \lambda v\\ &\implies U_B^{(t)} {\Sigma^2}^{(t)} U_B^{(t)\top} v = \lambda v \\
    &\implies (B^{(t)} B^{(t)\top}) U_B^{(t)} {\Sigma^2}^{(t)} U_B^{(t)\top} v = (B^{(t)} B^{(t)\top}) \lambda v \\
    &\implies U_B^{(t)} {\Sigma^2}^{(t)} U_B^{(t)\top} v = \lambda B^{(t)} B^{(t)\top} v.
\end{align*}
From the second and fourth rows, we obtain:
\begin{equation}
    \lambda B^{(t)} B^{(t)\top} v = \lambda v
\end{equation}
If $\lambda > 0$, we obtain that $v$ is an eigenvector of $B^{(t)} B^{(t)\top}$ and its corresponding eigenvalue is $=1$. If $\lambda = 0$, then we have:
\begin{align*}
    v \text{ is an eigenvector of } B^{(t)} B^{(t)\top}, \lambda = 0 &\implies B^{(t)} B^{(t)\top} v = 0\\
    &\implies U_B^{(t)} {\Sigma^2}^{(t)} U_B^{(t)\top} v = 0 \\
    &\implies v^\top U_B^{(t)} {\Sigma^2}^{(t)} U_B^{(t)\top} v = 0
\end{align*}
Since the matrix $B$ has rank $r$, it has exactly $r$ strictly positive singular values, while all other singular values beyond rank $r$ are zero. Thus, the matrix ${\Sigma^2}^{(t)}$ is positive definite. Hence, it must be:
\begin{equation}
    v^\top U_B^{(t)} {\Sigma^2}^{(t)} U_B^{(t)\top} v = 0 \implies U_B^{(t)\top} v = 0 \implies U_B^{(t)} U_B^{(t)\top} v = 0,
\end{equation}
which confirms that if $\lambda = 0$, then $v$ is an eigenvector of $U_B^{(t)} U_B^{(t)\top}$ with eigenvalue $0$.
\end{proof}

\tit{Generalization to Multiple Tasks.} This result generalizes directly to the multi-task setting, as defined in the main paper, by applying it to the matrix $\mathbf{B} \in \mathrm{R}^{m \times T \cdot r}$ obtained by horizontally stacking the LoRA updates $B^{(t)}$ from all tasks $T$. Then the optimal reference basis $U_*=\uref$ (as done in our approach) is given by the left singular vectors of $\mathbf{B}$, obtained via SVD. This matrix satisfies both orthonormality and the optimality conditions derived above. Defining $\Lambda_* =\mathrm{diag}(\lambda_1, \ldots, \lambda_{T \cdot r})$, where $\lambda_i$ are the eigenvalues of $\mathbf{B} \mathbf{B}^\top$, guarantees perfect reconstruction and zero alignment error for all tasks.

\subsection{Generalization Beyond the $T \cdot r \leq m, n$ Assumption}\label{appendix-sec:overcomplete}
The derivations in \cref{appendix-sec:least-squares,appendix-sec:error,appendix-sec:optimal-bases} assumed that the total LoRA rank $T \cdot r$ is less than both $m$ and $n$. We now show that this assumption is not necessary, and that \textbf{the reconstruction error remains zero even when $T \cdot r > m$ or $T \cdot r > n$}.

\tinytit{Intrinsic rank.} When $T \cdot r > m$ (for $B$) or $T \cdot r > n$ (for $A$), stacking the LoRA matrices still yields reference bases with intrinsic dimensions:
\begin{equation*}
d_U = \mathrm{rank}([B^{(1)}, \dots, B^{(T)}]) \leq m,
\qquad
d_V = \mathrm{rank}({[A^{(1)}, \dots, A^{(T)}]}^\top) \leq n,
\end{equation*}
since the number of linearly independent directions cannot exceed the number of rows or columns

 We can therefore replace the reference bases $\uref \in \mathbb{R}^{m \times T \cdot r}$ and $\vref \in \mathbb{R}^{n \times T \cdot r}$ with truncated orthonormal bases:
\begin{equation*}
U_B^{\text{ref}} \in \mathbb{R}^{m \times d_U},
\qquad
V_A^{\text{ref}} \in \mathbb{R}^{n \times d_V},
\end{equation*}
whose columns span the full LoRA update space.

\tinytit{Least-Squares Solution.}
Rewriting the least-squares problem (\cref{eq:app-least-squares}) in terms of the truncated $U_B^{\text{ref}}$ gives:
\begin{equation*}
R_B^{(t)} = \argmin_{R \in \mathbb{R}^{d_U \times r}} \big\| U_B^{\text{ref}} R - U_B^{(t)} \big\|_F^2.
\end{equation*}
Following the same derivation as in \cref{appendix-sec:least-squares}, the optimal solution is
\begin{equation*}
R_B^{(t)} = {(U_B^{\text{ref}})}^\top U_B^{(t)},
\end{equation*}
with an analogous expression for $Q_A^{(t)}$.

\tinytit{Alignment Error.}
Substituting this solution into the error expression of \cref{appendix-sec:error} shows that the optimal alignment error remains
\begin{equation*}
\varepsilon_U = r - \| {(U_B^{(t)})}^\top U_B^{\text{ref}} \|_F^2,
\end{equation*}
which achieves zero when $U_B^{\text{ref}}$ spans the column space of all $B^{(t)}$. Hence, the theoretical guarantees extend unchanged to the case $T \cdot r > m,n$.

\section{Computational Complexity Analysis}
\label{appendix-sec:computational}
\tit{Iso-C Complexity.} To assess the time complexity of our approach, we begin by analyzing that of Iso-C~\cite{marczak2025task}, to establish a baseline for comparison. For each layer, the Iso-C procedure can be broken down into the following steps:
\begin{itemize}
    \item \textit{LoRA $\rightarrow$ full space}: Compute $\Delta^{(t)} = B^{(t)} A^{(t)}$ for $t = 1, \dots, T$, given the matrices $B^{(t)}$ and $A^{(t)}$. The resulting complexity is $\mathcal{O}(T \cdot r \cdot n \cdot m)$. \item \textit{Summation}: Applying task arithmetic in the full space, $\Delta_{\operatorname{TA}} = \sum_t^T \Delta^{(t)}$, involves a cost of $\mathcal{O}(T \cdot n \cdot m)$.
    \item \textit{SVD computation}: Computing the decomposition $\Delta_{\operatorname{TA}} = U \Sigma V^\top$ for an $m \times n$ matrix has a complexity of $\mathcal{O}(m^2 \cdot n + n^3)$~\cite{cs357_svd_2020}.
    \item \textit{Isotropization}: The final step, $\Delta_{\operatorname{Iso-C}} = U \Sigma_{\operatorname{avg}} V^\top$, is dominated by $\mathcal{O}(m^2 \cdot n)$.
\end{itemize}

Overall, the total cost of Iso-C is dominated by $\mathcal{O}(T \cdot r \cdot n \cdot m + m^2 \cdot n + n^3)$. Assuming that $m = n$, then the time complexity of Iso-C is approximately cubic: $\mathcal{O}(n^3 + T \cdot r \cdot n^2)$ with respect to the number of features.

\tit{KnOTS Complexity.} Secondly, we analyze the computational cost of the \textbf{KnOTS} method~\cite{stoica2024knots}:

\begin{itemize}
    \item \textit{LoRA $\rightarrow$ full space}: As in Iso-C, this step has complexity $\mathcal{O}(T \cdot r \cdot n \cdot m)$.
    \item \textit{Concatenation}: Stacking all weight matrices $\Delta W = \left[ \Delta W^{(1)}, \dots, \Delta W^{(T)} \right]$ as block columns of a global matrix has a time complexity of $\mathcal{O}(T \cdot n \cdot m)$.
    \item \textit{SVD computation}: This is performed on a matrix of size $m \times (n \cdot T)$, resulting in a complexity of $\mathcal{O}(n^2 \cdot T^2 \cdot m + m^3)$, by making use of the transpose trick\footnote{$\operatorname{SVD}(P^\top) = U^{'} \Sigma^{'} {V^{'}}^\top \rightarrow P = (U^{'} \Sigma^{'} {V^{'}}^\top)^\top = {V^{'}} \Sigma {U^{'}}^\top$, thus, if $r \ll n$, this will reduce the number of operations. We will apply the transpose trick throughout.}.
    \item \textit{Merge}: Assuming simple task arithmetic is performed in the $V^\top$ space, $T$ blocks of $n$ columns each are summed, yielding a complexity of $\mathcal{O}(T^2 \cdot r \cdot n)$ to compute $V_{\operatorname{merge}}^\top$.
    \item \textit{Reconstruction}: The final step $\Delta_{\operatorname{KNOTS}} = U \Sigma V_{\operatorname{merge}}^\top$ has a complexity of $\mathcal{O}(m \cdot n \cdot T \cdot r + T \cdot r \cdot n)$.
\end{itemize}

Overall, the total cost of KNOTS is dominated by $\mathcal{O}(m^3 + T\cdot n(2r \cdot m + m + n \cdot T \cdot m + T \cdot r + r))$. Assuming $m = n$, the time complexity simplifies to $\mathcal{O}(n^3 T^2)$. Compared to Iso-C, the time complexity of KNOTS remains cubic with respect to the number of features. However, it includes an additional $T^2$ factor that scales quadratically with the number of tasks being merged.

\tit{TSV Complexity.} Then we analyze the cost of TSV~\cite{tsv}:
\begin{itemize}
    \item \textit{LoRA $\rightarrow$ full space}: As previously, this step has complexity $\mathcal{O}(T \cdot r \cdot n \cdot m)$.
    \item \textit{SVD Computation}: In this step SVD is performed on $T$ matrices of size $m\times n$ resulting in a complexity of $\mathcal{O}(T \cdot (m^2 \cdot n + n^3))$.
    \item \textit{Concatenation}: Stacking the first $k$ components of left and right singular vectors for all tasks results in $\mathcal{O}(T\cdot k (n+m))$.
    \item \textit{Global SVD Computation}: The SVD performed on the stacked matrices $n\times m$ requires $\mathcal{O}(2 \cdot (m^2 \cdot n + n^3))$.
    \item \textit{Obtaining orthogonal matrices}: This step requires $\mathcal{O}(2 \cdot m^2 \cdot n)$.
    \item \textit{Merge}: The final merge has a complexity of $\mathcal{O}(m \cdot n \cdot T \cdot r + T \cdot r \cdot n)$.
\end{itemize}
The overall cost of TSV is thus $\mathcal{O}(T \cdot r \cdot m \cdot n + T \cdot r \cdot n + T \cdot m^2 \cdot n + T \cdot n^3 +T \cdot k \cdot m + T \cdot k \cdot n +m^2 \cdot n + n^3)$. Assuming that $m=n$ and $T,r,k \ll n$, the computational cost is dominated by $\mathcal{O}(T\cdot n^3)$.

\tit{Core Space Complexity.} Finally, we analyze the cost of our approach in \textbf{Core Space}.
\begin{itemize}
  \item \textit{Stacking $A^{(t)}$ and $B^{(t)}$}: Stacking two sequences of $T$ matrices -- one with each matrix of shape $r \times n$ and the other with each matrix of shape $m \times r$ -- results in a cost of $\mathcal{O}(T \cdot r (n + m))$.
  \item \textit{SVD computation}: The stacked global $A$ matrix has shape $(T \cdot r) \times n$; hence, the cost of its SVD is $\mathcal{O}(n^2 \cdot (T \cdot r) + (T \cdot r)^3)$, by using the transpose trick. The stacked $B$ matrix has shape $m \times (T \cdot r)$, with a cost of $\mathcal{O}(m^2 \cdot (T \cdot r) + (T \cdot r)^3)$.  The overall cost of this step is $\mathcal{O}((m^2+n^2) \cdot (T \cdot r) + 2 \cdot (T \cdot r)^3)$.

  \item \textit{Low-rank loop}: In the optimized version of the low-rank loop, we only compute the matrix multiplication to obtain the aligned matrices. In this case the total cost is $\mathcal{O}(T \cdot (T m r^2 + T r^2 n + T^2 r^3))$ (using the optimal matrix multiplication order $(\uref B)(A \vref)$). Assuming that $T,r << n, m$, the cost is dominated by $\mathcal{O}(T^2 r^2 (m + n))$.

  \item \textit{Merge:} Assuming simple task arithmetic is performed in the aligned Core Space, the cost is $\mathcal{O}(T^3 \cdot r^2)$.
  \item \textit{Isotropization:} Optionally, Iso-C can be applied in the Core Space; since the Core Space is defined within a square matrix of dimension $T \cdot r \times T \cdot r$, this step adds an additional time complexity of $\mathcal{O}(T^3 r^3)$.
  \item \textit{Reconstruction:} The final step requires $\mathcal{O}(m \cdot T^2 \cdot r^2 + m \cdot T \cdot r \cdot n)$.
\end{itemize}

To sum up, our approach involves:
\begin{align*}
    \mathcal{O}(&\underbrace{T r (n + m)}_{\text{Stacking}} + \underbrace{T r(m^2+n^2) + 2 {(T r)}^3}_{\text{SVD refs.}} + \\
    & + \underbrace{(T^2 r^2 (m + n))}_{\text{Low-rank loop.}} + \underbrace{Tr(T^2 r + T r m + m  n)}_{\text{Merge \& Rec.}}) = \\
\mathcal{O}(&T r (m + n + 2m^2 + 2 n^2 + m n + r^2 ) + T^2 r^2 (2m + n + T) + T^3 r^3 )
\end{align*}

If we assume that $m = n$, the total time complexity simplifies to:
\begin{equation}
    \mathcal{O}(T r (2n + 5n^2 + r^2) + T^2 r^2 (3n + T) + T^3 r^3),
\end{equation}
which, if we assume $T,r \ll n$, is dominated by:
\begin{equation}
\mathcal{O}(T r n^2)
\end{equation}

\section{Additional Analysis}

\subsection{High subspace alignment leads to lower interference}
\label{sec:app_interference}
In this Section, we experimentally show that merging in Core Space reduces interference when merging models. We follow~\cite{yang2024representation, marczak2025task} and measure the interference as the L1 distance between the final embeddings of task-specific models and the merged one. We compare the interference when merging with TSV + Iso-C in Full Space versus Core Space. For each dataset, we collect the activations from the final layer (\ie, the projection to a common vision-language space) of both the task-specific model and the merged model. We present the average distance across all the samples in the test set. We observe lower interference when merging in Core Space, highlighting its effectiveness. Note that Full Space merging, which causes higher interference, also exhibits higher SAR in \cref{sec:analyses}.

\begin{figure}
    \centering
    \includegraphics[width=0.98\linewidth]{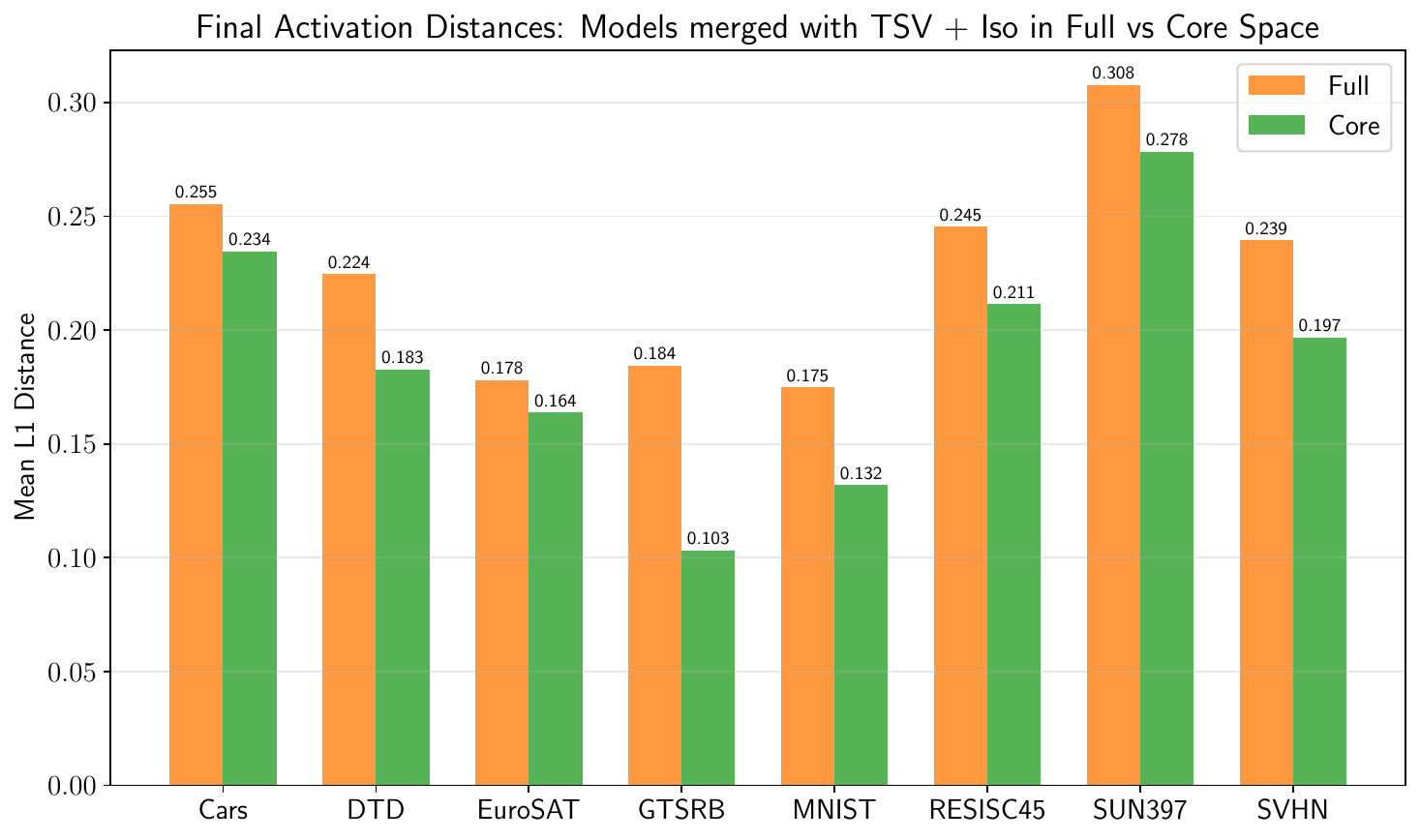}
    \caption{Mean L1 distance between the final embeddings of task-specific models and the merged one using TSV + Iso in Full and Core Space. We used ViT-B/16 model.}
    \label{fig:activations_l1_dist}
\end{figure}

\subsection{Rank of the merged update matrices}
Consider $T=8$ ViT-B/32 models fine-tuned with LoRA of rank $r=16$. The merged update matrices $\Delta W$ resulting from different merging methods and spaces can have different effective ranks $r_{\Delta W}$. Table~\ref{tab:rank_merged} reports the average rank of $\Delta W$ across all layers.

\begin{table}[h]
\centering
\caption{Average rank $r_{\Delta W}$ of merged update matrices $\Delta W$ obtained by merging 8 ViT-B/32 LoRA models with $r=16$ (so $Tr=128$).}
\label{tab:rank_merged}
\begin{tabular}{lccc}
\toprule
\textbf{Merging Space} & \textbf{TA} & \textbf{TSV} & \textbf{TIES} \\
\midrule
Full   & 128.00 & 128.00 & 766.25 \\
KnOTS  & 128.00 & 128.00 & 128.00 \\
Core   & 128.00 & 128.00 & 128.00 \\
\bottomrule
\end{tabular}
\end{table}

In most cases, the target rank of $\Delta W$ is equal to $Tr = 128$. The only exception is merging with TIES in Full space, where for weight matrices $W \in \mathbb{R}^{m \times n}$ the effective rank approaches the dimensionality of the matrices $d = \min(m, n) = 768$. This phenomenon arises because TIES performs trimming on the reconstructed weight matrices $\Delta W_t = BA$, which destroys the low-rank structure. In contrast, both Core and KnOTS operate directly in a constrained $Tr$-dimensional space, ensuring that the merged $\Delta W$ maintains the intended rank.
\section{Additional Experiment Details}\label{appendix-sec:exp-details}

\subsection*{Licenses of Used Datasets and Models}

In our research, we employed publicly available datasets and models, each governed by specific licenses. Below, we outline the sources and associated licenses for each:

\begin{itemize}
    \item \textbf{KnOTS LoRA Checkpoints}~\cite{stoica2024knots}: The KnOTS repository, which provides LoRA-adapted model checkpoints and training scripts, is licensed under the MIT License. This permissive license allows for reuse and modification with proper attribution.

    \item \textbf{Cars196}~\cite{krause20133d}: The Cars196 dataset is available for non-commercial research purposes. Specific licensing details are not explicitly provided.

    \item \textbf{Describable Textures Dataset (DTD)}~\cite{cimpoi_describing_2014}: The DTD is made available to the computer vision community for research purposes. The dataset is licensed under the Creative Commons Attribution 4.0 License (CC BY 4.0).

    \item \textbf{EuroSAT}~\cite{helber2018introducing}: The EuroSAT dataset is licensed under the MIT License.

    \item \textbf{German Traffic Sign Recognition Benchmark (GTSRB)}~\cite{stallkamp2012man}: The GTSRB dataset is licensed under the Creative Commons Zero (CC0) Public Domain Dedication.

    \item \textbf{MNIST}~\cite{lecun2010mnist}: The MNIST dataset is publicly available for research purposes. Specific licensing details are not explicitly provided; users are advised to consult the dataset's source for more information.

    \item \textbf{NWPU-RESISC45}~\cite{cheng_remote_2017}: The NWPU-RESISC45 dataset is licensed under the Creative Commons Attribution 4.0 License (CC BY 4.0).

    \item \textbf{SUN397}~\cite{sun397}: The SUN397 dataset is available for research purposes only. Specific licensing details are not explicitly provided; users are advised to consult the dataset's source for more information.

    \item \textbf{Street View House Numbers (SVHN)}~\cite{svhn}: The SVHN dataset is available for non-commercial use only.

    \item \textbf{Stanford Natural Language Inference (SNLI)}~\cite{bowman2015large}: The SNLI dataset is licensed under the Creative Commons Attribution-ShareAlike 4.0 International License (CC BY-SA 4.0).

    \item \textbf{Multi-Genre Natural Language Inference (MNLI)}~\cite{williams2017broad1coverage}: The MNLI dataset is released under the Open American National Corpus (OANC) license, which permits free use, modification, and sharing under permissive terms.

    \item \textbf{Sentences Involving Compositional Knowledge (SICK)}~\cite{marelli2014sick}: The SICK dataset is distributed under a Creative Commons Attribution-NonCommercial-ShareAlike license.

    \item \textbf{Question Natural Language Inference (QNLI)}~\cite{wang2018glue}: The QNLI dataset is part of the GLUE benchmark. Specific licensing details are not explicitly provided; users are advised to consult the dataset's source for more information.

    \item \textbf{Recognizing Textual Entailment (RTE)}~\cite{wang2018glue}: The RTE dataset is part of the GLUE benchmark. Specific licensing details are not explicitly provided; users are advised to consult the dataset's source for more information.

    \item \textbf{SciTail}~\cite{scitail}: The SciTail dataset is licensed under the Apache License 2.0.

\end{itemize}

\subsection{Experimental Environment}
The language experiments with Llama 3 8B were performed with a single 48G NVIDIA L40S. In contrast, the more affordable vision experiments were executed using a single 16G NVIDIA RTX 4080. To keep things fair, the reported times for the language experiments all refer to experiments performed on the same machine.

Our implementation builds directly on the KnOTS codebase~\cite{stoica2024knots} and uses the exact LoRA checkpoints they released. For full details on the original training and adaptation procedures, please refer to~\cite{stoica2024knots}.

\subsection{Hyperparameter Search}
To find optimal hyperparameters for each model, we adopt the widely used validation holdout strategy~\cite{stoica2024knots,marczak2025task,tsv,yadav2023tiesmerging}. Specifically, we perform a linear search for hyperparameters on the validation set, starting from a defined minimum value and incrementally increasing it until performance declines, indicating the optimal range. The identified optimal hyperparameters are then applied to the test set. We use the following search settings:
\begin{itemize}
    \item Scaling factor $\alpha$ starts at $0.1$, increasing in increments of $0.1$. This is used for every approach.
    \item The top-$K$ parameter for TIES and DARE-TIES begins at $10$ and increases in increments of $10$.
    \item The pruning factor $p$ for DARE-TIES starts at $0.1$ and increases in increments of $0.1$.
    \item For CART, the pruning rank is searched over the set $\{0.04, 0.08, 0.16, 0.32\}$, following the methodology of the original paper. Additionally, CART includes an extra scaling factor $\lambda$ in its merging formulation. Specifically, the merged weights are computed as $W_{\text{merged}}=W_0 + \alpha (\theta_{\text{avg}} + \lambda \sum_{t=1}^{T} \Bar{\tau}_t)$, where $\theta_{\text{avg}}$ denotes the average of the updates and $\Bar{\tau}_t$ represents the centered task vector for task $t$. For further details, we refer the reader to~\cite{choi2024revisiting}.
\end{itemize}

In \cref{tab:corespace-hypers}, we report the parameters used for the various merging methods in Core Space across all backbones. Note that we search for the optimal parameters for all methods across all spaces to maintain fairness. For the natural-language inference experiments, we omit CART as its hyperparameter tuning proved prohibitively expensive and exclude \iso\ variants, as they consistently degraded performance.
\begin{table}[ht]
\centering
\caption{Optimal hyperparameters for Core-Space merging on each backbone and merging strategy, including \iso\ variants.}
\label{tab:corespace-hypers}
\begin{tabular}{l l c c c c c}
\toprule
\textbf{Backbone} & \textbf{Merging Method}      & $\alpha$ & Top-\(K\) & Pruning \(p\) & CART rank & $\lambda$ \\
\midrule
\multirow{8}{*}{ViT-B/32}
  & TIES-Core            & 0.6 & 10      & -    & -    & -   \\
  & TIES-Core\iso        & 2.0 & 30     & -    & -    & -   \\
  & DARE-TIES-Core       & 0.6 & 10      & 0.1  & -    & -   \\
  & DARE-TIES-Core\iso   & 2.0 & 30   & 0.1  & -    & -   \\
  & TSV-Core             & 0.2 & -       & -    & -    & -   \\
  & TSV-Core\iso         & 0.9 & -       & -    & -    & -   \\
  & CART-Core            & 0.4 & -       & -    & 0.32 & 5.8 \\
  & CART-Core\iso        & 0.7 & -       & -    & 0.04 & 2.6 \\
  & Iso-C-Core           & 0.9 & - & -              & -            & -            \\

\midrule
\multirow{8}{*}{ViT-L/14}
  & TIES-Core            & 0.4 & 10      & -    & -    & -   \\
  & TIES-Core\iso        & 2.4 & 20  & -    & -    & -   \\
  & DARE-TIES-Core       & 0.4 & 10      & 0.1  & -    & -   \\
  & DARE-TIES-Core\iso   & 2.4 & 20  & 0.2  & -    & -   \\
  & TSV-Core             & 0.2 & -       & -    & -    & -   \\
  & TSV-Core\iso         & 0.9 & -       & -    & -    & -   \\
  & CART-Core            & 0.1 & -       & -    & 0.04 & 6.5 \\
  & CART-Core\iso        & 1.0 & -       & -    & 0.08 & 2.0 \\
  & Iso-C-Core     & 0.9 & -              & -              & -            & -            \\
\midrule
\multirow{4}{*}{Llama 3 8B}
  & TIES-Core            & 1.1 & 80      & -    & -    & -   \\
  & DARE-TIES-Core       & 1.1 & 80      & 0.1  & -    & -   \\
  & TSV-Core             & 0.5 & -       & -    & -    & -   \\
  & Iso-C-Core      & 2.8 & -              & -              & -            & -            \\
\bottomrule
\end{tabular}
\end{table}

\subsection{Code Availability and Pseudocode}
Our Core Space merging implementation is released at \url{https://github.com/apanariello4/core-space-merging}.

\Cref{listing:core-space} gives PyTorch-style pseudocode illustrating how to apply any merging strategy within the Core Space framework.
\begin{listing}
\begin{minted}[
framesep=2mm,
baselinestretch=1.1,
bgcolor=LightGray,
linenos]{python}
from torch.linalg import svd

A_list, B_list = ..., ... 

r, n = A_list[0].shape
m, _ = B_list[0].shape
 
A_stack = torch.cat(A_list, dim=0)  # (T*r, n)
B_stack = torch.cat(B_list, dim=1)  # (m, T*r)

# Calculate reference bases
Vh_A_ref = svd(A_stack, full_matrices=False)[2]  # (T*r, n)
U_B_ref = svd(B_stack, full_matrices=False)[0]   # (m, T*r)
 
M_list = []
for A, B in zip(A_list, B_list):
     M_aligned = (U_B_ref.T @ B) @ (A @ Vh_A_ref.T)
     M_list.append(M_aligned)
 
if merge_strategy == 'TA':
	M_merged = torch.stack(M_list).sum(dim=0)
elif merge_strategy == 'ties':
	M_merged = ties_merging(M_list)
elif merge_strategy == '...':
	M_merged = ...

# Reconstruct delta W
delta_W = U_B_ref @ M_merged @ Vh_A_ref
\end{minted}
\caption{Basic PyTorch pseudocode for model merging in Core Space.}
\label{listing:core-space}
\end{listing}

%
%
%
%

\section{Additional Results}

\subsection{Per-task evaluation in vision setting for ViT-L/14}\label{appendix-sec:vitl}
We provide in \cref{tab:per-task-vitl} per-task results vision model merging using the ViT-L/14 backbone. Similarly to what we observed for other backbones, in this case, performing the merging in Core Space yields consistent improvements across all methods, resulting in new state-of-the-art results.
\begin{table}[t]
\centering
\caption{Accuracies of merged models normalized against fine-tuned models on the vision datasets with ViT-L/14.}
\resizebox{\textwidth}{!}{
\begin{tabular}{lL{1.5cm}ccccccccc}
\toprule
\textbf{Method} & \textbf{Space} & \textbf{Cars} & \textbf{DTD} & \textbf{EuroSAT} & \textbf{GTSRB} & \textbf{MNIST} & \textbf{RESISC} & \textbf{SUN397} & \textbf{SVHN} & \textbf{Avg ($\Delta$ Acc)} \\
\midrule
\textit{Abs. Accuracy} & & 99.70 & 70.00 & 98.50 & 97.20 & 99.50 & 95.70 & 79.60 & 97.7 & -  \\
\midrule
TA & Full         & 80.01 & 79.50 & 65.59 & 59.98 & 82.20 & 79.55 & 86.71 & 64.74 & 74.79 \deltazero\\ 
\midrule
\multirow{3}{*}{TIES} & Full       & 79.65 & 78.28 & 64.43 & 61.10 & 83.82 & 79.42 & 87.45 & 69.94 & 75.51 \deltazero \\
& KnOTS  & 82.47 & 80.26 & 64.65 & 68.85 & 88.48 & 82.37 & 88.18 & 76.63 & \textbf{78.99} \textbf{\deltapos{+3.48}} \\ 
\ourrow
\cellcolor{white} & \textbf{Core}  & 81.94 & 80.03 & 65.78 & 68.94 & 87.44 & 81.85 & 88.48 & 73.75 & 78.53 \deltapos{+3.02} \\
\midrule
\multirow{3}{*}{DARE-TIES} & Full      & 79.70 & 78.82 & 64.99 & 60.63 & 83.92 & 79.32 & 87.07 & 69.84 & 75.53 \deltazero \\
& KnOTS & 80.41 & 79.65 & 64.84 & 62.95 & 85.33 & 79.73 & 87.55 & 71.00 & 76.43 \deltapos{+0.90}\\
\ourrow
\cellcolor{white} & \textbf{Core} & 82.97 & 80.03 & 65.66 & 68.34 & 87.75 & 82.48 & 88.80 & 75.88 & \textbf{78.99} \textbf{\deltapos{+3.46}} \\
\midrule
\multirow{3}{*}{TSV} & Full        & 82.38 & 80.11 & 66.12 & 68.18 & 85.46 & 83.02 & 87.89 & 70.76 & 77.99 \deltazero\\
& KnOTS   & 80.17 & 79.95 & 66.94 & 60.24 & 83.76 & 79.35 & 86.85 & 65.59 & 75.36 \deltaneg{-2.63} \\
\ourrow
\cellcolor{white} & \textbf{Core}    & 82.10 & 79.80 & 67.05 & 66.92 & 87.50 & 82.32 & 87.71 & 71.89 & \textbf{78.16} \textbf{\deltapos{+0.17}} \\
\midrule
\multirow{3}{*}{CART} & Full         &  91.88 & 88.53 & 75.51 & 80.88 & 68.99 & 92.77 & 88.17 & 64.81 & 81.44 \deltazero\\
& KnOTS    & 79.65 & 79.73 & 64.39 & 58.28 & 80.42 & 78.52 & 86.52 & 63.24 & 73.84 \deltaneg{-7.60} \\
\ourrow
\cellcolor{white} & \textbf{Core} & 86.02 & 86.33 & 72.39 & 82.82 & 91.03 & 85.93 & 88.46 & 72.67 & \textbf{83.21} \textbf{\deltapos{+1.77}} \\ 
\midrule
\multirow{3}{*}{TIES \iso} & Full & 84.11 & 83.07 & 74.94 & 76.66 & 92.88 & 87.88 & 88.58 & 74.61 & 82.84 \deltazero\\
& KnOTS & 86.44 & 88.23 & 78.74 & 78.94 & 94.27 & 87.94 & 88.62 & 72.24 & 84.43 \deltapos{+1.59} \\ 
\ourrow
\cellcolor{white} & \textbf{Core} & 91.13 & 90.58 & 79.53 & 87.67 & 86.47 & 90.46 & 90.05 & 72.28 & \textbf{86.02} \textbf{\deltapos{+3.18}} \\
\midrule
\multirow{3}{*}{DARE-TIES \iso} & Full & 83.25 & 81.78 & 73.03 & 77.25 & 86.81 & 86.78 & 88.24 & 74.44 & 81.45 \deltazero\\
& KnOTS & 87.32 & 87.78 & 74.12 & 81.50 & 94.21 & 89.48 & 88.88 & 70.87 & 84.27 \deltapos{+2.82} \\
\ourrow
\cellcolor{white} & \textbf{Core} & 90.84 & 91.12 & 79.15 & 88.14 & 90.19 & 90.73 & 89.92 & 72.05 & \textbf{86.52} \textbf{\deltapos{+5.07}} \\
\midrule
\multirow{3}{*}{TSV \iso} & Full & 86.44 & 89.07 & 82.49 & 84.68 & 90.76 & 90.03 & 87.99 & 67.98 & 84.93 \deltazero\\
& KnOTS  & 88.47 & 90.58 & 77.69 & 83.91 & 87.48 & 90.00 & 88.76 & 67.49 & 84.30 \deltaneg{-0.63}\\
\ourrow
\cellcolor{white} & \textbf{Core} & 91.54 & 91.34 & 80.24 & 86.79 & 87.39 & 91.51 & 89.59 & 71.30 & \textbf{86.21} \textbf{\deltapos{+1.28}} \\
\midrule
\multirow{3}{*}{CART \iso} & Full & 88.78 & 90.43 & 78.59 & 87.04 & 91.96 & 90.96 & 89.32 & 75.18 & 86.53 \deltazero\\ 
& KnOTS  & 88.72 & 89.60 & 80.77 & 79.84 & 87.18 & 89.50 & 88.38 & 68.07 & 84.01 \deltaneg{-2.52} \\ 
\ourrow
\cellcolor{white} & \textbf{Core} & 92.08 & 92.48 & 81.22 & 88.96 & 89.80 & 91.97 & 89.81 & 73.56 & \textbf{87.49} \textbf{\deltapos{+0.96}} \\ 
\midrule
\multirow{3}{*}{Iso-C} & Full & 86.83 & 86.94 & 80.65 & 77.99 & 92.09 & 87.88 & 88.50 & 68.69 & 83.70 \deltazero\\
& KnOTS & 88.27 & 89.75 & 78.36 & 85.41 & 91.65 & 90.93 & 88.85 & 70.97 & 85.52 \deltapos{+1.82}\\
\ourrow
\cellcolor{white} & \textbf{Core}    & 91.23 & 90.28 & 80.28 & 85.29 & 89.71 & 90.96 & 89.58 & 70.66 & \textbf{86.00} \textbf{\deltapos{+2.30}} \\
\bottomrule
\end{tabular}}
\label{tab:per-task-vitl}
\end{table}

\subsection{Experiments with Heterogeneous Ranks}
\label{appendix-sec:heterogeneous-ranks}

In the main paper, we discussed that Core Space merging naturally extends to the heterogeneous rank setting (\cref{tab:per-task-vitb-mixed}). Here we provide additional details.

When tasks are fine-tuned with different LoRA ranks, we horizontally and vertically stack the $B^{(t)}$ and $A^{(t)}$ matrices across tasks as usual. The resulting aggregate matrices have rank equal to the dimension of the union of all task subspaces. Performing SVD on these aggregates yields orthonormal reference bases $\uref$ and $\vref$ that span the combined subspaces, regardless of how individual task ranks vary.

Projection into these reference bases followed by task-specific alignment (see \cref{appendix-sec:least-squares}) guarantees that reconstruction is lossless. This explains why the results in \cref{tab:per-task-vitb-mixed} show that heterogeneous ranks incur no additional degradation in performance under our framework, whereas baselines that lack such alignment struggle with mismatched subspace dimensions.

\begin{table}[t]
\centering
\caption{Normalized accuracies of merged models on the vision tasks with ViT-B/32 with LoRA mixed ranks 16 (Cars, EuroSAT, MNIST, SVHN) and 64 (DTD, GTSRB, RESISC, SUN397).}
\resizebox{\textwidth}{!}{
\begin{tabular}{lL{1.5cm}ccccccccc}
\toprule
\textbf{Method} & \textbf{Space} & \textbf{Cars} & \textbf{DTD} & \textbf{EuroSAT} & \textbf{GTSRB} & \textbf{MNIST} & \textbf{RESISC} & \textbf{SUN397} & \textbf{SVHN} & \textbf{Avg ($\Delta$ Acc)} \\
\midrule
\textit{Abs. Accuracy} &  & 74.00 & 68.03 & 99.00 & 98.00 & 99.30 & 93.85 & 70.85 & 96.20 & -\\

\midrule
TA & - & 81.97 & 65.83 & 47.06 & 51.66 & 57.18 & 72.91 & 89.75 & 48.37 & 64.34 \deltazero \\
\midrule
\multirow{3}{*}{TIES}  & Full & 80.94 & 64.74 & 42.01 & 53.91 & 55.41 & 72.16 & 89.76 & 49.09 & 63.50 \deltazero \\
& KnOTS   & 82.70 & 67.08 & 49.01 & 55.00 & 57.39 & 73.74 & 90.14 & 45.97 & 65.13 \deltapos{+1.63}\\
\ourrow
\cellcolor{white} & \textbf{Core}  & 79.74 & 67.71 & 41.38 & 71.86 & 69.59 & 75.02 & 91.43 & 67.97 & 70.59 \textbf{\deltapos{+7.09}} \\
\midrule
\multirow{3}{*}{DARE-TIES} & Full      & 80.81 & 65.60 & 41.11 & 54.89 & 56.14 & 72.45 & 89.74 & 49.70 & 63.81 \deltazero \\
& KnOTS & 82.98 & 67.08 & 45.57 & 56.34 & 62.34 & 75.65 & 90.06 & 53.53 & 66.69 \deltapos{+2.88}\\
\ourrow
\cellcolor{white} & \textbf{Core}  & 80.98 & 68.80 & 44.26 & 68.07 & 64.43 & 74.45 & 91.44 & 62.99 & 69.43 \textbf{\deltapos{+5.62}} \\
\midrule
\multirow{3}{*}{TSV} & Full      & 80.98 & 67.94 & 50.95 & 59.10 & 64.65 & 75.82 & 90.76 & 53.37 & 67.95 \textbf{\deltazero} \\
& KnOTS & 82.51 & 65.91 & 43.66 & 48.98 & 60.52 & 71.47 & 89.22 & 51.29 & 64.20 \deltaneg{-3.75} \\
\ourrow
\cellcolor{white} & \textbf{Core}  & 80.46 & 69.51 & 51.44 & 58.33 & 61.03 & 76.70 & 89.37 & 52.43 & 67.41 \deltaneg{-0.54}\\
\midrule
\multirow{3}{*}{TIES \iso} & Full & 78.94 & 69.90 & 54.02 & 53.13 & 66.83 & 75.11 & 90.64 & 46.63 & 66.90 \deltazero \\
& KnOTS & 81.42 & 68.49 & 60.19 & 45.09 & 56.70 & 70.83 & 90.43 & 40.11 & 64.16 \deltaneg{-2.74} \\
\ourrow
\cellcolor{white} & \textbf{Core} & 79.32 & 75.84 & 58.66 & 64.09 & 80.14 & 76.76 & 90.52 & 55.13 & 72.56 \textbf{\deltapos{+5.66}}\\
\midrule
\multirow{3}{*}{DARE-TIES \iso} & Full & 80.33 & 70.76 & 60.98 & 51.48 & 61.58 & 74.16 & 90.41 & 47.26 & 67.12 \deltazero \\
& KnOTS & 82.11 & 67.71 & 58.92 & 41.44 & 57.92 & 69.69 & 90.37 & 38.79 & 63.37 \deltaneg{-3.75}\\
\ourrow
\cellcolor{white} & \textbf{Core} & 78.75 & 75.37 & 58.47 & 64.98 & 80.83 & 77.00 & 90.49 & 55.53 & 72.68 \textbf{\deltapos{+5.56}} \\
\midrule
\multirow{3}{*}{TSV \iso} & Full  & 76.46 & 73.65 & 61.02 & 56.41 & 69.02 & 74.04 & 90.01 & 49.12 & 68.72 \deltazero \\
& KnOTS & 78.88 & 72.79 & 66.59 & 54.61 & 78.83 & 72.21 & 89.03 & 50.28 & 70.40 \deltapos{+1.68}\\
\ourrow
\cellcolor{white} & \textbf{Core} & 79.30 & 78.26 & 63.94 & 58.11 & 77.28 & 74.40 & 89.85 & 50.97 & 71.51 \textbf{\deltapos{+2.79}}\\
\midrule
\multirow{3}{*}{Iso-C} & Full        & 78.48 & 75.61 & 63.37 & 63.05 & 77.14 & 77.10 & 90.22 & 51.50 & 72.06 \deltazero \\
& KnOTS  & 78.92 & 76.31 & 56.19 & 63.50 & 75.11 & 77.08 & 90.53 & 52.44 & 71.26 \deltaneg{-0.80}\\
\ourrow
\cellcolor{white} & \textbf{Core}    &  81.17 & 79.20 & 59.71 & 70.86 & 81.90 & 78.93 & 90.81 & 56.60 & 74.90 \textbf{\deltapos{+2.84}}\\
\bottomrule
\end{tabular}}\label{tab:per-task-vitb-mixed}
\end{table}

\subsection{Extension to VeRA}\label{appendix-sec:vera}

We also evaluated the applicability of Core Space merging beyond LoRA, specifically on VeRA~\cite{kopiczko2023vera} (\cref{tab:per-task-vitbr16-vera}). In VeRA, the decomposition $\Delta W = \Lambda_b B \Lambda_d A$ differs structurally from LoRA since $A$ and $B$ are fixed random matrices, and only the scaling vectors $\Lambda_b, \Lambda_d$ are trainable.

To apply our method, we absorb the scaling vectors into the low-rank matrices:
\begin{equation*}
\Tilde{B} = \Lambda_b B, \qquad \Tilde{A} = \Lambda_d A,
\end{equation*}
and then treat $(\Tilde{A}, \Tilde{B})$ as if they were standard LoRA components. Since the subsequent steps (stacking, SVD, projection, and alignment) are agnostic to how $A$ and $B$ were obtained, the derivations in \cref{appendix-sec:no-info-loss} apply without modification.

The empirical results in \cref{tab:per-task-vitbr16-vera} confirm this reasoning: Core Space merging consistently outperforms other approaches even in the VeRA setting, validating that the framework is general to low-rank adaptation methods beyond LoRA.

\begin{table}[t]
\centering
\caption{Normalized accuracies of merged models on the vision tasks with ViT-B/32 with VeRA rank 16.}
\resizebox{\textwidth}{!}{
\begin{tabular}{lL{1.5cm}ccccccccc}
\toprule
\textbf{Method} & \textbf{Space} & \textbf{Cars} & \textbf{DTD} & \textbf{EuroSAT} & \textbf{GTSRB} & \textbf{MNIST} & \textbf{RESISC} & \textbf{SUN397} & \textbf{SVHN} & \textbf{Avg ($\Delta$ Acc)} \\
\midrule
\textit{Abs. Accuracy} &  & 62.79 & 57.07 & 96.55 & 90.85 & 98.60 & 88.50 & 62.79 & 93.10 & - \\
\midrule
TA & - & 95.78 & 77.73 & 47.26 & 39.77 & 49.10 & 70.61 & 100.74 & 37.29 & 64.78 \deltazero \\
\midrule
\multirow{3}{*}{TIES} & Full &  95.53 & 77.73 & 44.76 & 39.07 & 48.25 & 69.92 & 100.58 & 35.62 & 63.93 \deltazero \\
& KnOTS   &  96.64 & 77.91 & 50.13 & 39.61 & 50.46 & 70.52 & 100.85 & 35.93 & 65.26 \deltapos{+1.33} \\
\ourrow
\cellcolor{white} & \textbf{Core}  & 97.06 & 77.73 & 44.30 & 41.98 & 50.35 & 71.63 & 100.44 & 38.95 & 65.31 \textbf{\deltapos{+1.39}} \\
\midrule
\multirow{3}{*}{DARE-TIES} & Full      & 94.44 & 77.73 & 47.26 & 42.09 & 49.82 & 69.37 & 100.05 & 38.28 & 64.88 \deltazero\\
& KnOTS &  96.52 & 77.54 & 52.13 & 40.88 & 50.46 & 70.25 & 100.06 & 37.24 & 65.63 \textbf{\deltapos{+0.75}}\\
\ourrow
\cellcolor{white} & \textbf{Core}  & 97.09 & 77.82 & 44.27 & 42.02 & 50.41 & 71.75 & 100.48 & 38.97 & 65.35 \deltapos{+0.47}\\
\midrule
\multirow{3}{*}{TSV}  & Full       &  93.35 & 77.54 & 45.11 & 37.44 & 53.50 & 69.40 & 99.70 & 33.85 & 63.74 \deltazero\\
& KnOTS   & 92.73 & 78.19 & 54.85 & 38.77 & 56.20 & 69.10 & 98.82 & 34.02 & 65.33 \deltapos{+1.59} \\
\ourrow
\cellcolor{white} & \textbf{Core}     &  93.75 & 76.51 & 60.15 & 37.94 & 54.32 & 68.38 & 98.77 & 34.63 & 65.56 \textbf{\deltapos{+1.82}}\\
\midrule
\multirow{3}{*}{TIES \iso} & Full &  94.79 & 77.26 & 46.84 & 36.11 & 48.78 & 68.19 & 100.81 & 34.03 & 63.35 \deltazero\\
& KnOTS &  94.81 & 77.26 & 47.33 & 35.49 & 48.67 & 68.10 & 100.67 & 33.74 & 63.26 \deltaneg{-0.09} \\
\ourrow
\cellcolor{white} & \textbf{Core} &  95.43 & 77.63 & 50.79 & 37.14 & 50.61 & 69.06 & 100.62 & 33.77 & 64.38 \textbf{\deltapos{+1,03}}\\
\midrule
\multirow{3}{*}{DARE-TIES \iso} & Full &  93.82 & 76.98 & 46.34 & 36.80 & 49.47 & 68.26 & 100.44 & 34.69 & 63.35 \deltazero\\
& KnOTS & 94.88 & 77.17 & 47.37 & 35.52 & 48.64 & 68.13 & 100.66 & 33.75 & 63.27 \deltaneg{-0.08}\\
\ourrow
\cellcolor{white} & \textbf{Core} & 94.14 & 77.26 & 53.01 & 37.45 & 52.90 & 68.13 & 100.09 & 31.80 & 64.35 \textbf{\deltapos{+1.00}}\\
\midrule
\multirow{3}{*}{TSV \iso} & Full  & 93.37 & 76.70 & 48.29 & 35.92 & 51.47 & 68.11 & 100.30 & 34.35 & 63.56 \deltazero\\
& KnOTS & 92.58 & 75.30 & 52.51 & 36.47 & 58.16 & 67.36 & 99.83 & 35.43 & 64.71 \deltapos{+1.15}\\
\ourrow
\cellcolor{white} & \textbf{Core} &  93.57 & 76.89 & 63.33 & 39.27 & 57.53 & 68.63 & 99.18 & 34.07 & 66.56 \textbf{\deltapos{+3.00}}\\
\midrule
\multirow{3}{*}{Iso-C} & Full        &  93.50 & 77.26 & 52.05 & 37.42 & 49.49 & 68.06 & 100.62 & 34.22 & 64.08 \deltazero\\
& KnOTS  &  94.59 & 77.35 & 47.76 & 35.52 & 48.73 & 68.27 & 100.52 & 33.68 & 63.30 \deltaneg{-0.78} \\
\ourrow
\cellcolor{white} & \textbf{Core}    &  91.77 & 75.77 & 63.60 & 38.53 & 58.53 & 67.81 & 97.67 & 36.50 & 66.27 \textbf{\deltapos{+2.19}}\\
\bottomrule
\end{tabular}}\label{tab:per-task-vitbr16-vera}
\end{table}

\end{document}